\documentclass{article}

\usepackage[preprint]{neurips_2026}

\usepackage[utf8]{inputenc} 
\usepackage[T1]{fontenc}    
\usepackage{hyperref}       
\usepackage{url}            
\usepackage{booktabs}       
\usepackage{amsfonts}       
\usepackage{nicefrac}       
\usepackage{microtype}      
\usepackage{xcolor}         
\usepackage{todonotes}

\usepackage{amsmath, amsthm, thmtools, thm-restate, }

\usepackage{etoolbox} 
\makeatletter
\newcommand{\nonumberthanks}[1]{%
  \begingroup
  \renewcommand{\thefootnote}{}
  \renewcommand{\@makefnmark}{}
  \footnotetext{#1}%
  \endgroup
}
\makeatother

\newtheorem{theorem}{Theorem}[section]

\newtheorem{lemma}[theorem]{Lemma}

\theoremstyle{remark}
\newtheorem{remark}[theorem]{Remark}
\theoremstyle{definition}

\newtheorem{assumption}{Assumption}

\DeclareMathOperator*{\argmax}{arg\,max}

\DeclareMathOperator*{\lin}{lin}

\DeclareMathOperator*{\diag}{diag}

\DeclareMathOperator{\softmax}{smx}
\DeclareMathOperator{\softmaxr}{sxr}
\DeclareMathOperator*{\interior}{int}
\DeclareMathOperator*{\op}{op} 
\DeclareMathOperator*{\const}{const}

\DeclareMathOperator*{\Ent}{H}
\DeclareMathOperator*{\MSE}{MSE}
\DeclareMathOperator*{\CE}{CE}
\DeclareMathOperator*{\CEr}{CEr}
\DeclareMathOperator*{\BS}{BS}
\DeclareMathOperator*{\prob}{p}
\DeclareMathOperator*{\Ens}{Ens}
\DeclareMathOperator*{\Lap}{Lap}

\newcommand{\norm}[1]{\left\lVert#1\right\rVert}

\title{The Neural Tangent Kernel for Classification}

\author{%
Jonathan Plenk$^{1,2,*}$,
 Sergio Calvo-Ordoñez$^{1,2,3,*}$, Álvaro Cartea$^{1,2}$,\\
  \textbf{Yarin Gal$^{3}$, Mark van der Wilk$^{4}$, Kamil Ciosek$^{5}$} \\
  $^1$Mathematical Institute, University of Oxford \\
  $^2$Oxford-Man Institute of Quantitative Finance, University of Oxford \\
  $^3$OATML, University of Oxford\\
  $^4$Department of Computer Science, University of Oxford\\
  $^5$Spotify \\
}

\begin{document}
\nonumberthanks{* Equal contribution.}

\maketitle

\begin{abstract}
In wide neural networks, the Neural Tangent Kernel (NTK) remains approximately constant during training, providing a powerful theoretical tool for studying training dynamics, generalization, and connections to kernel methods. However, this theory is largely restricted to regression losses. It was previously thought that training on a classification loss, or more generally losses involving nonlinear output transformations, breaks this property, leading to divergent logits and a breakdown of the linearization. In this paper, we extend NTK theory to classification by identifying conditions under which wide neural networks remain in the lazy training regime. We show that parameter-space regularization ensures a constant NTK during training for cross-entropy loss, while in the absence of regularization the regime is recovered when targets are non-degenerate, i.e. when all classes have strictly positive probability. Under these conditions, training is well-approximated by the linearized model, yielding an explicit characterization of the solution in terms of the NTK. We further analyze the distribution of trained predictors induced by random initialization and relate this notion of model uncertainty to Bayesian methods.
\end{abstract}

\section{Introduction}

The theoretical understanding of deep learning has significantly advanced through the study of wide neural networks via the neural tangent kernel (NTK) \citep{jacot2018neural}. In the infinite-width limit, the NTK remains approximately constant during training, and the network is well-approximated by its parameter linearization \citep{lee2019wide}. This perspective enables a precise analysis of training dynamics and generalization \citep{jacot2018neural, arora2019exact, du2019gradient}, and establishes a close connection between neural networks, kernel methods \citep{cao2019generalization}, and Gaussian processes \citep{lee2019wide, calvo2025observation}. In particular, for mean-squared error (MSE) loss, the trained network exactly fits the training data and admits a tractable Gaussian distribution on test inputs, determined by the random initialization \citep{he2020bayesian}. While this framework has been extensively developed for regression, its applicability beyond this setting remains unexplored.

For classification objectives such as cross-entropy, the logits that perfectly fit the training data diverge, causing the parameters to leave the lazy training regime \citep{yu2025divergence}. As a result, the NTK is no longer constant and the linearization breaks down, preventing a direct extension of NTK theory to classification. \cite{calvo2025observation} showed for MSE loss that parameter-space regularization around initialization preserves the NTK regime and admits a Bayesian interpretation in terms of observation noise. Whether a similar mechanism can recover NTK behavior for classification losses remains unclear, since cross-entropy training is known to produce divergent logits and a breakdown of linearization. In this work, we show that parameter-space regularization is sufficient to recover lazy training for cross-entropy and more general function-space losses, yielding approximation by the linearized model under mild assumptions.


\citet{liu2020linearity} show that when a nonlinearity is placed in the last layer, the NTK is no longer constant during training, an example being a network with softmax outputs. Our key contribution here is to, instead, treat the softmax as part of the function-space loss, which allows us to recover the NTK regime at the level of the pre-activation outputs. In addition to the regularized setting, we consider the case without parameter-space regularization. We show that when the target probabilities are non-degenerate, i.e. assign strictly positive probability to every class, the pre-activation outputs remain in the NTK regime and are well-approximated by their linearization.

Building on these theoretical results, we study the distribution over trained predictors induced by random initialization. This notion of model uncertainty has been previously analyzed in the NTK regime \citep{he2020bayesian} for MSE with extra label noise, where it is shown to be more conservative than Bayesian last-layer methods, but less conservative than the NTK-GP posterior. For general loss functions, the training dynamics of the linearized network remain nonlinear. Nevertheless, we show that the stationary point admits an explicit characterization via an inverse map in function-space. This characterization enables efficient sampling from the trained network, which we use to approximate the distribution over predictors. Empirically, we find that this closely matches the ensemble distribution obtained by retraining the network from different initializations.


Our main contributions can be summarized as follows:
\begin{itemize}
    \item We show that parameter-space regularization preserves the NTK regime for cross-entropy and more general losses, extending NTK theory beyond regression.
    
    \item We show that, even without regularization, the NTK regime is recovered for classification when target probabilities are non-degenerate, covering settings such as label smoothing.
    
    \item We characterize the distribution over trained predictors induced by random initialization, and relate this notion of model uncertainty to Bayesian methods.
    
    \item We provide empirical validation supporting the theoretical predictions.
\end{itemize}

\section{Preliminaries}
\subsection{Wide neural networks in NTK parametrization}
Consider a standard feedforward neural network $f_{\theta}$ with inputs $x$ in a compact space $M^d\subseteq \mathbb{R}^d$, parameters $\theta \in \mathbb{R}^p$, and outputs $f_{\theta}(x) \in \mathbb{R}^K$. For smooth activations, $\theta\mapsto f_{\theta}(x)$ is twice continuously differentiable. For any points $\mathbf{x}_1,\ldots, \mathbf{x}_N \in M^d$ define $f_{\theta}(\mathbf{x}) \in \mathbb{R}^{NK}$ by element-wise application, and similarly the parameter-Jacobian $J_{\theta}(\mathbf{x}) := \nabla_{\theta} f_{\theta}(\mathbf{x})^{\top} \in \mathbb{R}^{NK\times p}$ and parameter-Hessian $H_{\theta}(\mathbf{x}) := \nabla_{\theta\theta}^2 f_{\theta}(\mathbf{x}) \in \mathbb{R}^{NK \times p \times p}$. If the last layer is linear, \citet{liu2020linearity} show that the parameter-Jacobian is locally bounded and $\mathcal{O}\left((\log n)^c/\sqrt{n}\right)$-Lipschitz for large layer width $n$. We restate this in the following lemma, and provide more explanations in Appendix \ref{app: ReminderNTKTheory}.
\begin{restatable}{lemma}{JacobianboundedLipschitz}
\label{lemma: Jacobian Lipschitz Hessian bounded}
For any $\delta_0>0$ there are $K'_1,K'_2>0$ such that: For every radius $R>0$ there is large enough layer width $n$ such that with probability $1-\delta_0$ over random initialization $\theta_0$: For any input $x \in M^d$:
\begin{equation}
    \forall \theta \in B(\theta_0,R): \left\lVert J_{\theta}(x)\right\rVert_{2,2} \le K'_1,
\end{equation}
\begin{equation}
    \forall \theta \in B(\theta_0,R): \left\lVert H_{\theta}(x) \right\rVert_{\op,2} 
    \le \frac{(\log n)^c}{\sqrt{n}} K'_2.
\end{equation}
Recall that $J_{\theta}(x) \in \mathbb{R}^{K\times p}$ and $H_{\theta}(x) \in \mathbb{R}^{K\times p\times p}$. We are using the Euclidean norm on the $K$-dimensional output dimension, after applying the Euclidean and spectral norm on the parameter dimensions $p$ and $p\times p$ respectively.
The constants $K_1', K_2' > 0$ do not depend on $R$. The constant $c > 0$ only depends on the network architecture.
\end{restatable}

\citet{lee2018dnnsgps, matthews2018gaussian} show that the network output at initialization converges in distribution for infinite width:
\begin{restatable}{lemma}{NNGPconvergence}
\label{lemma: NNGPconvergence}
Consider random initialization $\theta_0$. Then $f_{\theta_0}(\cdot)$ converges in distribution to a Gaussian process with zero mean and covariance given by the NNGP Kernel $\mathcal{K}$: For inputs $\mathbf{x}_1,\ldots,\mathbf{x}_N \in M^d$,
\begin{equation}
    f_{\theta_0}(\mathbf{x}) \overset{d}{\longrightarrow} \mathcal{N}(0,\mathcal{K}(\mathbf{x}, \mathbf{x})).
\end{equation}
\end{restatable}
Our results apply to any parametric model with the properties of these two lemmas, which have been proven for more general architectures in \citep{yang2020tensorprograms2}.
\subsection{The training loss in function-space}
Consider training points $\mathbf{x}_1,\ldots,\mathbf{x}_N \in M^d$. Consider a function-space loss
\begin{equation}
    \mathcal{C}: \mathbb{R}^{NK} \to \mathbb{R}, 
    \quad \mathbf{z} \mapsto \mathcal{C}(\mathbf{z})
    ,
\end{equation}
which we will evaluate in the function outputs
\begin{equation}
    \mathbf{z} 
    := \mathbf{f}
    := f(\mathbf{x})
    := (f(\mathbf{x}_1), \ldots, f(\mathbf{x}_N)) \in \mathbb{R}^{NK}.
\end{equation}
We assume $\mathcal{C}$ is twice continuously differentiable, and denote the function-space gradient by $
\nabla_{\mathbf{z}}\mathcal{C}(\mathbf{z}) 
\in \mathbb{R}^{NK}$ and the function-space Hessian by $\nabla^2_{\mathbf{z}\mathbf{z}}\mathcal{C}(\mathbf{z}) \in \mathbb{R}^{NK\times NK}$.



Consider a classification task with $K$ classes. Denote the target probabilities for $N$ training points by $\mathbf{p}_1,\ldots,\mathbf{p}_N \in \Delta^{K-1}$. One-hot encodings correspond to $\mathbf{p}_i = e_{\mathbf{y}_i}$ for training labels $\mathbf{y}_i \in \{1,\ldots, K \}$.
The function-space categorical cross-entropy loss (Appendix \ref{appx: cce}) is
\begin{equation}
    \mathcal{C}_{\CE}(\mathbf{z}) 
    := \sum_{i=1}^N\CE[\mathbf{p}_i, \softmax(\mathbf{z}_i)]
    = -\sum_{i=1}^N \sum_{k=1}^{K} \mathbf{p}_{ik} \log\softmax(\mathbf{z}_i)_k
    .
\end{equation}
Here $\softmax:\mathbb{R}^K\to\interior(\Delta^{K-1})$ denotes the $\mathrm{softmax}$-function. To avoid redundancies from $\softmax(z) = \softmax(z + \alpha 1)$ we can alternatively use $\mathrm{softmax}$ with a fixed reference class (Appendix \ref{appx: cerc}).

\subsection{Regularized gradient flow}
Consider a regularization strength $\beta \ge 0$.  Define the regularized parameter-space loss
\begin{equation}
    \mathcal{L}^{\beta}: \mathbb{R}^p \to \mathbb{R},
    \quad
    \mathcal{L}^{\beta}(\theta) := \mathcal{C}(f_{\theta}(\mathbf{x})) + \frac{1}{2}\beta \left\lVert \theta - \theta_0 \right\rVert_2^2.
\end{equation}
If $\beta=0$ we assume $\mathcal{C}$ is lower bounded, i.e. $\inf\mathcal{C}>-\infty$.
Starting from random initial parameters $\theta_0\in\mathbb{R}^p$ this objective is minimized using the gradient flow\footnote{Picard-Lindelöf gives the existence of a unique solution on $[0,T_{\max})$ for some $T_{\max}>0$ with $\lim_{t\uparrow T_{\max}} \lVert \theta_t \rVert_2 = \infty$ if $T_{\max}<\infty$. We will prove $T_{\max}=\infty$. With $t\ge 0$ we refer to $t\in [0,T_{\max})$.}
\begin{equation}
    \frac{d}{dt}\theta_t
    = -\eta_0 \nabla_{\theta} \mathcal{L}^{\beta}(\theta_t)
    = -\eta_0  \left( \nabla_{\theta} \mathcal{C}(\mathbf{f}_{\theta_t})+ \beta (\theta_t- \theta_0) \right).
\end{equation}
By the chain rule,
\begin{equation}
    \nabla_{\theta} \mathcal{C}(f_{\theta}(\mathbf{x})) 
    = J_{\theta}(\mathbf{x})^{\top} \nabla_{\mathbf{z}}\mathcal{C}(f_{\theta}(\mathbf{x}))
    \in \mathbb{R}^p.
\end{equation}
\begin{restatable}{lemma}{boundinitialloss}
    \label{lemma: boundinitialloss}
Let $\delta_0>0$ be arbitrarily small.
There is $K_0>0$, such that for large enough layer width $n$, with probability $1-\delta_0$ over random initialization $\theta_0$:
\begin{equation}
    \mathcal{C}(f_{\theta_t}(\mathbf{x}))
    \le  \mathcal{L}^{\beta}(\theta_t)
    \le \mathcal{L}^{\beta}(\theta_0)
    = \mathcal{C}(f_{\theta_0}(\mathbf{x})) \le K_0
    \quad
    \text{for all $t\ge 0$}.
\end{equation}
\end{restatable}
The proof is in Appendix \ref{appendix: functionspaceloss}. Given $K_0$, we define the function-space loss sublevel set
\begin{equation}
    \mathcal{S}_0 
    := \{\mathbf{z}\in\mathbb{R}^{NK}: \mathcal{C}(\mathbf{z}) \le K_0 \}.
\end{equation}

\subsection{Parameter-linearization of a neural network}

The parameter-linearization of a neural network around its initialization $\theta_0\in\mathbb{R}^p$ is
\begin{equation}
    f_{\theta}^{\lin}(x) := f_{\theta_0}(x) + J_{\theta_0}(x)(\theta - \theta_0).
\end{equation}
This model remains nonlinear in the input $x$, but is linear in the parameters $\theta$, which simplifies the gradient flow dynamics
\begin{equation}
    \frac{d}{dt}\theta_t^{\lin}
    = -\eta_0 \left(J_{\theta_t^{\lin}}^{\lin}(\mathbf{x})^{\top} \nabla_{\mathbf{z}} \mathcal{C}\left(f^{\lin}_{\theta_t^{\lin}}(\mathbf{x})\right)+ \beta (\theta_t^{\lin} - \theta_0)\right),
\end{equation}
as the parameter-Jacobian remains constant: $J_{\theta_t^{\lin}}^{\lin}(\mathbf{x})= J_{\theta_0}(\mathbf{x})$.
For any input $x \in M^d$, multiplying both sides with $J_{\theta_0}(x)$ gives the function-space ODE of the linearized network
\begin{equation}
    \frac{d}{dt} f^{\lin}_{\theta_t^{\lin}}(x)
    = -\eta_0 \left( J_{\theta_0}(x)J_{\theta_0}(\mathbf{x})^{\top}  \nabla_{\mathbf{z}} \mathcal{C}\left(f^{\lin}_{\theta_t^{\lin}}(\mathbf{x})\right) + \beta\left( f^{\lin}_{\theta^{\lin}_t}(x) - f_{\theta_0}(x) \right)\right).
\end{equation}
We see that the dynamics of the linearized network only depend on the parameters $\theta_t^{\lin}$ through its function values $\hat{g}_t := f_{\theta_t^{\lin}}^{\lin}$. They are driven by the empirical NTK at initialization $\hat{\Theta}_{\theta_0}(x,\mathbf{x}) = J_{\theta_0}(x)J_{\theta_0}(\mathbf{x})^{\top} \in \mathbb{R}^{K\times NK}$. We define the function-space values $g$ driven by the analytical NTK $\Theta$ through the ODEs
\begin{equation}
\label{eq: fctspaceodean}
    \frac{d}{dt} g_t(x) = -\eta_0 \left(\Theta_{x,\mathbf{x}} \nabla_{\mathbf{z}} \mathcal{C}\left( g_t(\mathbf{x}) \right) + \beta\left( g_t(x) - g_0(x) \right)\right),
    \quad
    g_0(x) = f_{\theta_0}(x).
\end{equation}

\section{Theoretical results for general loss functions}\label{sec: theory-general-loss}

\subsection{Approximation by parameter-linearization}
In this section, we establish conditions under which wide neural networks can be approximated by their parameter linearization along the training path. The approximation error vanishes as the width increases, at a rate of order $O((\log n)^c / \sqrt{n})$. We will prove this as a consequence of lazy training.
\begin{theorem}
\label{thm: beta=0 lazy training}
Let $\delta_0>0$ be arbitrarily small.
For $\beta>0$ assume that $\mathcal{C}$ is convex, and that there is $K_1>0$ such that
\begin{equation}
    \forall \mathbf{z}\in\mathcal{S}_{0}:
    \quad
    \lVert \nabla_{\mathbf{z}}\mathcal{C}(\mathbf{z}) \rVert_2 \le K_1.
\end{equation}
For $\beta=0$ assume that the Hessian of $\mathcal{C}$ is strictly positive definite
and that there are $K_2, \mu_{\mathcal{C}}>0$, such that
\begin{equation}
\forall \mathbf{z}\in\mathcal{S}_{0}:
\quad
2\mu_{\mathcal{C}} \left(\mathcal{C}(\mathbf{z}) - \inf \mathcal{C} \right)
\le 
    \left\lVert \nabla_{\mathbf{z}} \mathcal{C}(\mathbf{z}) \right\rVert_2^2
    \le 2K_2 \left(\mathcal{C}(\mathbf{z}) - \inf \mathcal{C} \right).
\end{equation}
Then we can uniformly approximate the network by its linearization during training. This means that there is $C_2>0$ such that for large enough layer width $n$, with probability $1-\delta_0$ over random initialization $\theta_0$,
\begin{equation}
\label{eq: linearapprox}
    \forall x\in M^d, t\ge 0: 
    \quad
    \left\lVert f_{\theta_t}(x) - f_{\theta_t^{\lin}}^{\lin}(x)\right\rVert_2 
    \le C_2 \frac{(\log n)^c}{\sqrt{n}}.
\end{equation}
\end{theorem}
We discuss the assumptions in Appendix \ref{appendix: functionspaceloss}, and prove them for various function-space losses in Appendix \ref{appendix: functionspacelossexamples}.

The proof for $\beta>0$ is in Appendix \ref{appendix: exponentialdecayregularizer}. The assumptions are satisfied by the categorical cross-entropy loss $\mathcal{C}_{\CE}$ (with and without a reference class, see Appendix \ref{appx: cce} and \ref{appx: cerc}) for general target distributions $\mathbf{p}_i$, including one-hot targets. This contrasts with the unregularized case, where divergence of the logits prevents the network from remaining in the NTK regime \citep{yu2025divergence}. We will see in the experimental section how even a small parameter regularizer ensures the NTK regime.

The proof for $\beta=0$ is in Appendix \ref{appendix: exponentialdecayPL}. When target distributions $\mathbf{p}_i$ have full support, the assumptions of the theorem are satisfied by the categorical cross-entropy loss $\mathcal{C}_{\CEr}$ with a reference class (see Appendix \ref{appx: cerc}). This corresponds to label smoothing \citep{szegedy2016rethinking, muller2019does}.

The (standard) categorical cross-entropy loss $\mathcal{C}_{\CE}$ without a reference class is convex, but its Hessian is nowhere strictly positive definite 
as it is invariant with respect to shifts. In Appendix \ref{appendix: extensionsoftmaxnoreference} we extend Theorem \ref{thm: beta=0 lazy training} for $\beta=0$ to show that still (when target distributions have full support), the centered version of the network\footnote{Let $P := I_K - \frac{1}{K}1_K1_K^{\top} \in \mathbb{R}^{K\times K}$ denote the orthogonal projection onto $1_K^{\perp}$. The centered network output is obtained by applying $P$ to the logits, i.e. $P f_\theta(x)$.} is well approximated by the centered version of its linearization during training. This means that
\begin{equation}
    \forall x\in M^d, t\ge 0: 
    \quad
    \left\lVert P \left(f_{\theta_t}(x) - f_{\theta_t^{\lin}}^{\lin}(x) \right) \right\rVert_2 
    \le C_2 \frac{(\log n)^c}{\sqrt{n}}.
\end{equation}
Due to $\softmax(z) = \softmax(Pz)$, this implies that 
\begin{equation}
\label{eq: linearizedpreact}
    \forall x \in M^d,t\ge 0:
    \quad
    \left\lVert \softmax\left( f_{\theta_t}(x) \right) - \softmax\left( f_{\theta_t^{\lin}}^{\lin}(x) \right) \right\rVert_2 \le \frac{1}{2} C_2 \frac{(\log n)^c}{\sqrt{n}}.
\end{equation}

The closeness to the linearization is particularly novel in this setting. \citet{liu2020linearity} studied networks with a nonlinear last layer,
\begin{equation}
    \tilde{f}_{\theta}(x) := \softmax(f_{\theta}(x)),
\end{equation}
and show that the NTK of $\tilde{f}$ is no longer constant during training, preventing approximation by direct linearization of $\tilde{f}$. In contrast, we treat the last nonlinearity as part of the loss, while keeping the network architecture with a linear last layer. Equation \ref{eq: linearizedpreact} shows that $\tilde{f}$ can be approximated by applying the last nonlinearity to the linearized pre-activation.

\subsection{The function-space training dynamics of the linearized network}

In the previous section we characterized the conditions under which wide neural networks are well approximated by their linearized version during training. We now study the gradient flow dynamics of the linearized network, whose stationary point provides an explicit characterization of the trained predictor in function-space.

Consider the kernel regression problem
\begin{equation}
    \min_{g\in \mathcal{H}_{\Theta}} \mathcal{C}(g(\mathbf{x})) + \frac{1}{2}\beta \left\lVert g - f_{\theta_0} \right\rVert_{\Theta^{-1}}^2
    \quad
    \text{or}
    \quad
    \min_{\mathbf{z}\in\mathbb{R}^{NK}} \mathcal{C}(\mathbf{z}) + \frac{1}{2}\beta (\mathbf{z}-\mathbf{f}_{\theta_0})^{\top} \Theta_{\mathbf{x},\mathbf{x}}^{-1} (\mathbf{z}-\mathbf{f}_{\theta_0}),
\end{equation}
via the kernel representation $g(x) = f_{\theta_0}(x) + \Theta_{x,\mathbf{x}}\Theta_{\mathbf{x},\mathbf{x}}^{-1} (\mathbf{z} - \mathbf{f}_{\theta_0})$. The optimal function is equal to the stationary point $g_{\infty}$ of the function-space ODE (\ref{eq: fctspaceodean}). It satisfies
\begin{equation}
\label{eq: stationarypoint}
   \Theta_{x,\mathbf{x}} \nabla_{\mathbf{z}} \mathcal{C}\left(g_{\infty}(\mathbf{x})\right) + \beta g_{\infty}(x)
    = \beta f_{\theta_0}(x).
\end{equation}
Restricting to the training points $x = \mathbf{x}_1,\ldots,\mathbf{x}_N$, define the function-space operator
\begin{equation}
    \Phi: \mathbb{R}^{NK} \to \mathbb{R}^{NK},
    \quad
    \Phi(\mathbf{z}) 
    := \Theta_{\mathbf{x},\mathbf{x}} \nabla_{\mathbf{z}}\mathcal{C}(\mathbf{z}) + \beta \mathbf{z}.
\end{equation}
The stationary point equation (\ref{eq: stationarypoint}) implies:
\begin{lemma}
The predictions of the linearized network in the training points $\mathbf{x}_1,\ldots,\mathbf{x}_N$ are
\begin{equation}
    g_{\infty}(\mathbf{x})
    = \Phi^{-1}(\beta f_{\theta_0}(\mathbf{x})).
\end{equation}
\end{lemma}
Using this, we can compute the trained values at any test point $x$:
\begin{restatable}{lemma}{linearizedpredinvertibleNTK}
\label{lemma: linearizedpredinvertibleNTK}
Let $\beta \ge 0$ and assume $\Theta_{\mathbf{x},\mathbf{x}}$ is invertible. Then the prediction of the linearized network at any test point $x$ is
\begin{equation}
    g_{\infty}(x)
    = f_{\theta_0}(x) + \Theta_{x,\mathbf{x}} \Theta_{\mathbf{x},\mathbf{x}}^{-1} \left(\Phi^{-1}(\beta f_{\theta_0}(\mathbf{x})) - f_{\theta_0}(\mathbf{x}) \right).
\end{equation}
Moreover for $\beta>0$ the prediction can also be written as
\begin{equation}
    g_{\infty}(x)
    = f_{\theta_0}(x) - \frac{1}{\beta} \Theta_{x,\mathbf{x}} \nabla_{\mathbf{z}} \mathcal{C}\left(\Phi^{-1}(\beta f_{\theta_0}(\mathbf{x}))\right).
\end{equation}
\end{restatable}
The proof is in Appendix \ref{appendix: proofslinearizeddynamics}.

These expressions allow us to sample from the trained neural network $f_{\theta_{\infty}}(x)\approx g_{\infty}(x)$ by sampling from the random initialization $f_{\theta_0}((x,\mathbf{x})) \sim \mathcal{N}(0,\mathcal{K}_{(x,\mathbf{x}),(x,\mathbf{x})})$. This provides the distribution of an infinite-width ensemble of predictors induced by initialization, see Figure \ref{fig:toy_ensembles}. We analyze this notion of model uncertainty in the next section.
\begin{figure}[hbt!]
    \centering
    \includegraphics[width=\linewidth]{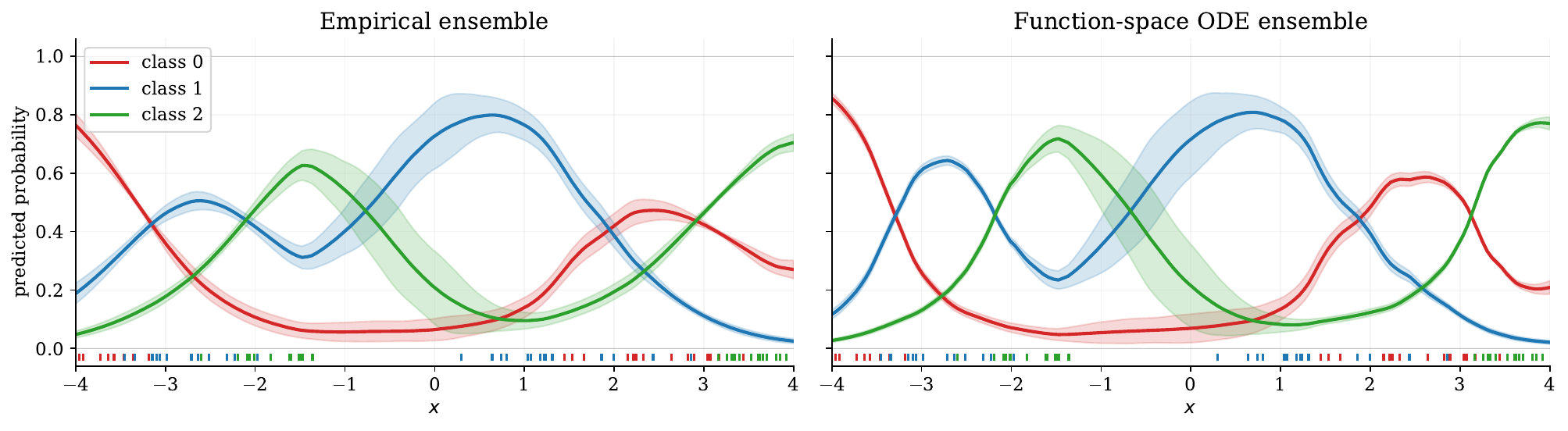}
    \caption{%
    $1$d-classification with $3$ classes.
         Left: An ensemble over wide networks, starting from different parameter initializations. Right: The infinite-width limit of the ensemble, using the function-space ODE.
    }
\label{fig:toy_ensembles}
\end{figure}


\section{Connection of the infinite-width ensemble to Bayesian methods}

The previous section characterized the trained linearized predictor through the inverse map $\Phi^{-1}$. We now use this characterization to study the distribution over trained predictors induced by random initialization. 

\subsection{Connection to MAP inference for classification}

Consider parameter regularizer $\beta>0$
and target probabilities $\mathbf{p}_1,\ldots,\mathbf{p}_N \in \Delta^{K-1}$. In this section we focus on one-hot targets $\mathbf{p}_i = e_{\mathbf{y}_{i}}$ for labels $\mathbf{y}_i\in \{1,\ldots,K \}$.

For categorical cross-entropy loss $\mathcal{C}_{\CE}$, the function-space operator $\Phi: \mathbb{R}^{NK} \to \mathbb{R}^{NK}$ is given by
\begin{equation}
    \Phi(\mathbf{z}) = \sum_{i=1}^{N} \Theta_{\mathbf{x},\mathbf{x}_i}  \left(\softmax(\mathbf{z}_i) - \mathbf{p}_i \right)  + \beta \mathbf{z}.
\end{equation}
For large layer width, Lemma \ref{lemma: linearizedpredinvertibleNTK} (or solving the function-space ODE (\ref{eq: fctspaceodean})) allows us to sample from $f_{\theta_{\infty}}$ by sampling from $f_{\theta_0}$. The resulting random variable $\softmax(f_{\theta_{\infty}}(x))$ is a sample probability vector from the simplex of discrete probability distributions on $\{1,\ldots,K \}$. This represents epistemic uncertainty for classification coming from initialization variance.  While this differs from Bayesian uncertainty (see Section \ref{sec: connectionlaplace}), we can still draw parallels by viewing each sample as a MAP estimate for a different prior mean.
Consider the probabilistic likelihood model
\begin{equation}
    y(x) = \argmax_{k=1,\ldots,K} \left( F(x)_k + \tau \epsilon_k \right),
    \quad
    \epsilon_k \sim \mathrm{Gumbel}(0,1).
\end{equation}
This has conditional likelihood \citep{maddison2014sampling}
\begin{equation}
    \prob\left(y(x)=k|F(x)\right) = \softmax\left(\frac{1}{\tau}F(x)\right)_k
    \propto
    \exp\left(-\frac{1}{\tau} \CE(e_k, F(x)) \right).
\end{equation}
Moreover, place the GP prior $F(x)\sim \mathcal N(F_0(x), \Theta_{x,x})$. The MAP in this model minimizes
\begin{equation}
    F\mapsto -\log \prob(F|(\mathbf{x}_i, \mathbf{y}_i)_{i=1}^N)
    = -\sum_{i=1}^N \log \softmax\left(\frac{1}{\tau}F(\mathbf{x}_i) \right)_{\mathbf{y}_i}+ \frac{1}{2}\left\lVert F - F_0\right\rVert_{\Theta^{-1}}^2
    + \const.
\end{equation}
Reparameterizing $F=\tau g$, $F_0=\tau f_{\theta_0}$, $\beta=\tau^2$, recovers the objective
\begin{equation}
    g
    \mapsto -\sum_{i=1}^N \log \softmax\left(g(\mathbf{x}_i) \right)_{\mathbf{y}_i}+ \frac{1}{2}\beta \left\lVert g - f_{\theta_0}\right\rVert_{\Theta^{-1}}^2
    + \const.
\end{equation}
This shows that in the infinite-width limit, the neural network solution converges to the temperature-rescaled MAP estimate in a probabilistic model whose observation noise is determined by $\beta$. The equivalence to the NTK-GP has been analyzed in \citet{calvo2025observation} for MSE loss corresponding to regression tasks. Our results allow the use of their tools for classification NTK-GP inference under arbitrary prior means via a single neural network pass.

\subsection{Connection of the ensemble variance to the Laplace approximation}
\label{sec: connectionlaplace}
The infinite-width ensemble distribution over $f_{\theta_{\infty}}(x)$ is intractable for general function-space losses $\mathcal{C}$. By linearizing $\nabla_{\mathbf{z}}\mathcal{C}(\mathbf{z})$ we can approximate it with a normal distribution: Let $\mathbf{g}^* := \Phi^{-1}(\mathbf{0})$ denote the MAP in the training points for prior mean $f_{\theta_0}=0$, i.e. the solution of 
\begin{equation}
\Theta_{\mathbf{x},\mathbf{x}} \nabla_{\mathbf{z}} \mathcal{C}(\mathbf{g}^*) + \beta \mathbf{g}^* = 0.
\end{equation}

Linearize the function-space gradient around $\mathbf{g}^*$:
\begin{equation}
\nabla_{\mathbf{z}} \mathcal{C}(\mathbf{z})
\approx \nabla_{\mathbf{z}} C(\mathbf{g}^*) + \nabla_{\mathbf{z}}^2 \mathcal{C}(\mathbf{g}^*)(\mathbf{z}-\mathbf{g}^*) 
= -\beta \Theta_{\mathbf{x},\mathbf{x}}^{-1} \mathbf{g}^* + \nabla_\mathbf{z}^2 \mathcal{C}(\mathbf{g}^*)(\mathbf{z} - \mathbf{g}^*).
\end{equation}
Define $\mathbf{H}^* = \nabla_{\mathbf{z}}^2 \mathcal{C}(\mathbf{g}^*)$ and $\mathbf{A} := \Theta_{\mathbf{x},\mathbf{x}}\mathbf{H}^* + \beta I_{NK}$. While $\mathbf{A}$ is not symmetric, $\mathbf{H}^* \mathbf{A}^{-1}$ is symmetric.
The linearization gives the approximate stationary point in the training points $\mathbf{x}$:
\begin{equation}
    g_{\infty}(\mathbf{x})
    \approx \mathbf{g}^* + \mathbf{A}^{-1} \beta f_{\theta_0}(\mathbf{x}).
\end{equation}
In the test points $\mathbf{x}'$:
\begin{equation}
    g_{\infty}(\mathbf{x}') 
    \approx f_{\theta_0}(\mathbf{x}')  + \Theta_{\mathbf{x}',\mathbf{x}} \Theta_{\mathbf{x},\mathbf{x}}^{-1} \mathbf{g}^* - \Theta_{\mathbf{x}',\mathbf{x}} \mathbf{H}^* \mathbf{A}^{-1} f_{\theta_0}(\mathbf{x}).
\end{equation}
This is linear in the initial function values $f_{\theta_0}((\mathbf{x}',\mathbf{x})) \sim \mathcal{N}(0, \mathcal{K}_{(\mathbf{x}',\mathbf{x}),(\mathbf{x}',\mathbf{x})})$, and allows us to approximate the infinite-width ensemble distribution by a Gaussian with mean and covariance
\begin{equation}
    \mu_{\Ens}(\mathbf{x}') 
    = \Theta_{\mathbf{x}',\mathbf{x}}\,\Theta_{\mathbf{x},\mathbf{x}}^{-1} \mathbf{g}^*,
\end{equation}
\begin{align}
\Sigma_{\Ens}(\mathbf{x}',\mathbf{x}') 
=& \mathcal{K}_{\mathbf{x}',\mathbf{x}'} - K_{\mathbf{x}',\mathbf{x}}  \mathbf{A}^{-\top} \mathbf{H}^* \Theta_{\mathbf{x},\mathbf{x}'} - \Theta_{\mathbf{x}',\mathbf{x}} \mathbf{H}^* \mathbf{A}^{-1} \mathcal{K}_{\mathbf{x},\mathbf{x}'} \\
&+ \Theta_{\mathbf{x}',\mathbf{x}}  \mathbf{A}^{-\top} \mathbf{H}^* \mathcal{K}_{\mathbf{x},\mathbf{x}}    \mathbf{H}^* \mathbf{A}^{-1} \Theta_{\mathbf{x},\mathbf{x}'}.
\end{align}
For $\mathcal{K}=\Theta$ the (approximate) ensemble covariance simplifies to
\begin{equation}
    \Sigma_{\Ens}(\mathbf{x}', \mathbf{x}') 
    = \Theta_{\mathbf{x}',\mathbf{x}'} 
    - \Theta_{\mathbf{x}',\mathbf{x}}\left( \mathbf{H}^* \mathbf{A}^{-1}  +\beta \mathbf{A}^{-\top} \mathbf{H}^* \mathbf{A}^{-1} \right)\Theta_{\mathbf{x},\mathbf{x}'}
\end{equation}
which is more uncertain than the $\Sigma_{\Ens}$ resulting from $\mathcal{K} \preceq \Theta$.

The Laplace approximation \citep{rasmussen2005gaussian} is a standard tool for approximate Bayesian inference in classification GPs, in which model uncertainty is captured by a Bayesian posterior, rather than an ensemble covariance. The Laplace posterior mean is the same as our approximate ensemble mean. The Laplace posterior covariance is
\begin{equation}
    \Sigma_{\Lap}(\mathbf{x}',\mathbf{x}') 
    = \Theta_{\mathbf{x}',\mathbf{x}'} - \Theta_{\mathbf{x}',\mathbf{x}} \mathbf{H}^{*} \mathbf{A}^{-1}   \Theta_{\mathbf{x},\mathbf{x}'}.
\end{equation}
For $\mathcal{K}=\Theta$ the gap to our approximate ensemble covariance is equal to
\begin{equation}
    \Sigma_{\Lap}(\mathbf{x}',\mathbf{x}') - \Sigma_{\Ens}(\mathbf{x}',\mathbf{x}')
    = \beta \Theta_{\mathbf{x}',\mathbf{x}} \mathbf{A}^{-\top} \mathbf{H}^* \mathbf{A}^{-1} \Theta_{\mathbf{x},\mathbf{x}'}
    \succeq 0.
\end{equation}
Hence the Laplace approximate NTK-GP-posterior covariance is always more uncertain than our approximate ensemble covariance.
This generalizes the analysis of \citet{he2020bayesian}, who consider MSE loss where the ensemble covariance and NTK-GP posterior covariance are available in closed form.


\section{Experiments}
\subsection{Classification toy dataset}
Following \citet{liu2020linearity} we consider a one-dimensional toy dataset with $3$ classes. As discussed in Section \ref{sec: theory-general-loss}, they show that when the last layer of the network is defined as $\mathrm{softmax}$, the NTK does not remain constant, even when applying label smoothing or parameter regularization. In Figure \ref{fig:toy_ntk_prepost} we show that the NTK of the pre-$\mathrm{softmax}$ layer does remain constant, allowing the linearization discussed in equation (\ref{eq: linearizedpreact}).

\begin{figure}[hbt!]
    \centering
    \includegraphics[width=\linewidth]{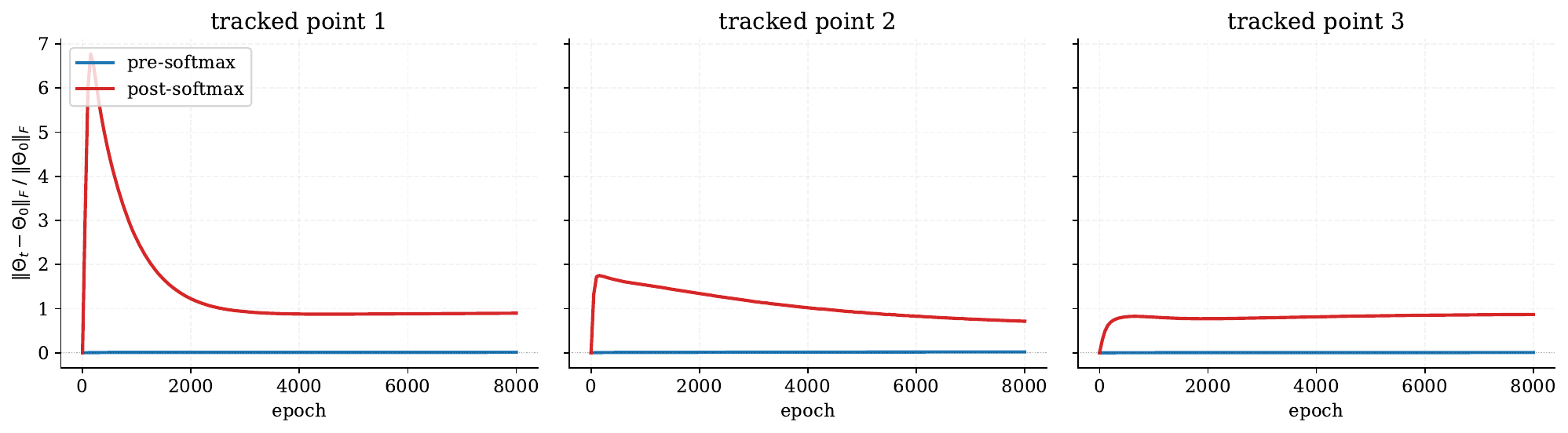}
    \caption{%
    Blue: Pre-$\mathrm{softmax}$ NTK (constant). Red: Post-$\mathrm{softmax}$ NTK (not constant).
    }
\label{fig:toy_ntk_prepost}
\end{figure}


\subsection{MNIST Classification}
Following \citet{yu2025divergence} we train a four-layer fully connected neural network on MNIST \citep{lecun2002gradient} using $2$ classes (odd or even). By using $\mathrm{softmax}$ with a reference class, the logit dimension is $1$ and thus the kernel is scalar-valued. We plot the evolution of the empirical NTK $t\mapsto \hat{\Theta}_{\theta_t}(x,x)$ during training in three input points $x\in \mathbb{R}^{784}$ in Figure \ref{fig:mnist-tracking}. In line with \citet{yu2025divergence} we see that the empirical NTK diverges for the classification task with one-hot encodings. When using label smoothing and zero parameter regularization, it converges but still moves nontrivially from initialization. This behavior can be expected as the required layer width $n$ in Theorem \ref{thm: beta=0 lazy training} grows inversely to the $\mu_{\mathcal{C}}$-constant of the function-space loss. When applying parameter regularization $\beta = 0.01$, the NTK remains almost constant during training, even for relatively low layer width. 

\begin{figure}[hbt!]
    \centering
    \includegraphics[width=\linewidth]{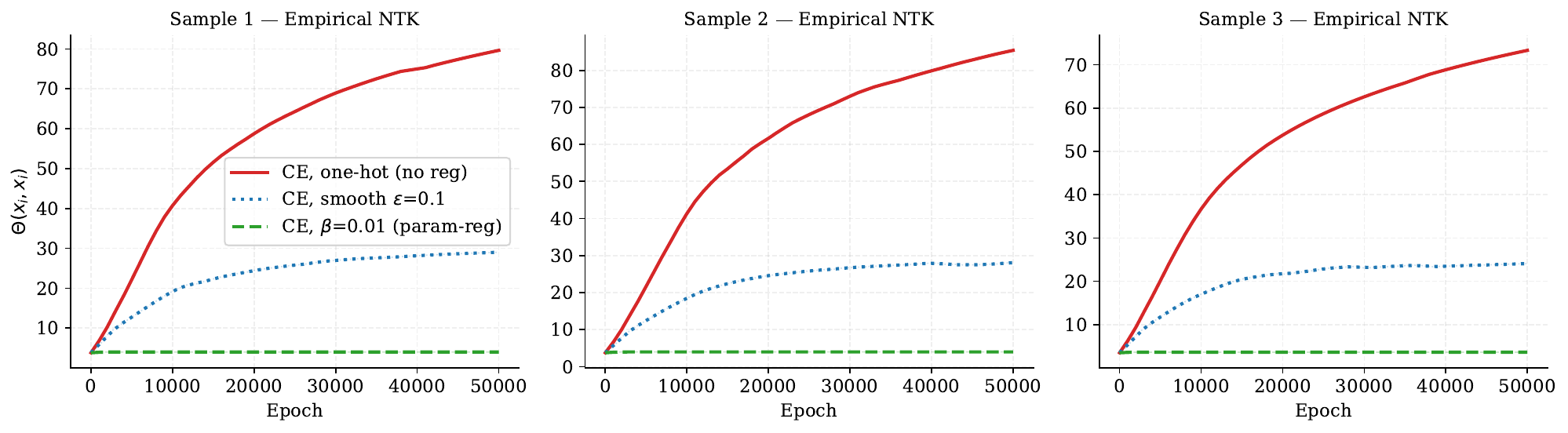}
    \caption{%
       The NTK for MNIST classification does not diverge when using label smoothing or parameter regularization.
    }
\label{fig:mnist-tracking}
\end{figure}

\section{Related Work}

A central obstruction arises when extending Neural Tangent Kernel (NTK) theory to classification. For cross-entropy or logistic losses on separable data, the optimum is attained only in the limit of diverging logits, and gradient descent drives parameters to infinity while converging only in direction \citep{soudry2018implicit, lyu2019gradient, ji2020directional}. This behavior violates the small-perturbation assumptions underlying NTK linearization and leads to a breakdown of kernel constancy. Recent work formalizes this incompatibility by showing that, under cross-entropy training, the empirical NTK can itself diverge over time rather than remain fixed \citep{yu2025divergence}. A related limitation arises when nonlinear output layers are incorporated into the model: even at initialization, such architectures need not admit a constant tangent kernel \citep{liu2020linearity}.

In response to these limitations, prior work has largely moved beyond the NTK regime rather than restoring it. One line of research studies the implicit bias of gradient-based optimization under cross-entropy, characterizing the asymptotic direction of the parameters rather than a finite predictor. This yields max-margin solutions in linear and homogeneous models \citep{soudry2018implicit, lyu2019gradient} and extends to deep settings, where it connects to neural collapse geometry \citep{papyan2020prevalence, ji2021unconstrained, thrampoulidis2022imbalance}. Recent work provides more detailed descriptions of cross-entropy dynamics beyond asymptotics, including tractable path-level analyses in non-convex models \citep{garrod2025diagonalizing}. A complementary line develops alternative infinite-width limits, such as mean-field and feature-learning regimes, which allow representations to evolve and depart from fixed-kernel behavior \citep{mei2018mean, yang2021tensor}. While these approaches capture phenomena beyond the lazy regime, they do not yield a finite-predictor NTK description of classification. More recently, NTK-based constructions have been used for uncertainty estimation in classification by leveraging empirical NTK features within scalable approximations \citep{calvo2026richer}, but without characterizing classification training dynamics within the NTK regime.

The present work addresses this gap by identifying conditions under which classification remains in a valid NTK regime. Rather than replacing cross-entropy or analyzing asymptotic directions, we show that the failure modes of NTK classification can be avoided by modifying the optimization landscape. In particular, regularization around initialization controls parameter drift, while full-support target distributions—such as those induced by label smoothing \citep{szegedy2016rethinking, muller2019does}—eliminate the infinite-logit optimum and yield finite solutions. Under these conditions, the relevant object is the pre-activation predictor, whose tangent kernel remains effectively constant throughout training. This yields an NTK characterization of the classification dynamics in function-space.

\section{Conclusion}

We extend the Neural Tangent Kernel (NTK) theory to classification and more general losses by identifying conditions under which wide neural networks remain in the lazy training regime. In particular, we show that parameter-space regularization and full-support target distributions prevent the divergence of logits and ensure that training is well-approximated by the linearized model. This yields an explicit NTK characterization of classification dynamics at the level of finite predictors. We further analyze the distribution over trained networks induced by random initialization, connecting this notion of uncertainty to Bayesian interpretations. Finally, we empirically validate our theoretical predictions through controlled experimental setups.
\bibliographystyle{plainnat}
\bibliography{references}

\appendix

\section{Recap of NNGP and NTK theory for wide neural networks}
\label{app: ReminderNTKTheory}
\subsection{Feedforward neural networks in NTK parametrization}

Consider an input $x^0 := x \in \mathbb{R}^d$. A neural network with $L$ hidden layers and linear last layer is defined through
\begin{equation}
    x^{(l)} := \phi_l\left(x^{(l-1)}, \theta^{(l)} \right)
    \quad \text{for} \quad
    l=1,\ldots,L,
\end{equation}
\begin{equation}
    f_{\theta}(x) := \frac{1}{\sqrt{n_{L}}} W^{(L+1)}x^{(L)} + b^{(L+1)}.
\end{equation}
Each hidden layer map $\phi_l: \mathbb{R}^{n_{l-1}} \times \mathbb{R}^{p_l} \to \mathbb{R}^{n_{l}}$ takes the previous layer $x^{(l-1)} \in \mathbb{R}^{n_{l-1}}$ and current parameters $\theta^{(l)} \in \mathbb{R}^{p_l}$ as input. The final layer parameters are $\theta^{L+1} = (W^{(L+1)}, b^{(L+1)}) \in \mathbb{R}^{p_{L+1}}$, where $W^{(L+1)} \in \mathbb{R}^{K\times n_L}$, $b^{(L+1)} \in \mathbb{R}^{K}$, $p_{L+1} = K (n_L +1)$. The full parameter vector is $\theta = \left(\theta^{(l)}\right)_{l=1}^{L+1} \in \mathbb{R}^{p}$ and the total number of parameters is $p = \sum_{l=1}^{L+1} p_l$.

Consider an activation function $\phi:\mathbb{R}\to\mathbb{R}$ that is applied element-wise. A standard feedforward layer in NTK parameterization is defined by
\begin{equation}
    \phi_l(x^{(l-1)}, \theta^{(l)}) = \phi\left(\frac{1}{\sqrt{n_{l-1}}} W^{(l)} x^{(l-1)} + b^{(l)}\right).
\end{equation}
Here, the parameters are $\theta^{(l)} = (W^{(l)}, b^{(l)}) \in \mathbb{R}^{n_{l}(n_{l-1} + 1)}$, where $W^{(l)} \in \mathbb{R}^{n_{l} \times n_{l-1}}$, $b^{(l)} \in \mathbb{R}^{n_{l}}$. They are randomly initialized through $W^{(l)}_{0,ij}\sim \mathcal{N}(0,\sigma^2_{w_l})$, $ b^{(l)}_{0,i} \sim \mathcal{N}(0,\sigma^2_{b_l})$  iid. Similarly one can define a residual layer and a convolutional layer. 

In the following, $f_{\theta}$ is a standard feedforward neural network whose non-polynomial activation function $\phi$ is twice continuously differentiable, and $\phi$ and $\phi'$ are globally Lipschitz.


\subsection{Local boundedness and Lipschitzness of the parameter-Jacobian}
\citet{liu2020linearity} prove for a standard feedforward neural network (as well as ResNets and CNNs with large number of channels):
\JacobianboundedLipschitz*
To gain intuition consider a one-hidden-layer neural network with one-dimensional inputs and outputs and no biases:
\begin{equation}
    f: \mathbb{R}\times \mathbb{R}^{2n} \to \mathbb{R}, 
    \quad
    f_{a,w}(x) = \frac{1}{\sqrt{n}} \sum_{j=1}^n a_j \phi(w_jx).
\end{equation}
As it is linear and therefore easy to analyze in the last-layer parameters $a$, we only look at the derivatives with respect to hidden weights $w$. The Jacobian and Hessian entries are
\begin{equation}
    \frac{\partial f_{\theta}(x)}{\partial w_j }  = \frac{1}{\sqrt{n}} a_j\phi'(w_jx)x, 
    \quad \frac{\partial^2 f_{\theta}(x)}{\partial w_j\partial w_k}  = \frac{1}{\sqrt{n}} a_j\phi''(w_jx)x^2 1_{\{j=k \}}.
\end{equation}
Consider the random initialization $a_j,w_j\sim \mathcal{N}(0,1)$ iid. Recall that $|\phi'|$ and $|\phi''|$ are globally bounded. The Euclidean norm of the Jacobian at initialization is
\begin{equation}
\left\lVert \frac{\partial f_{\theta}(x) }{\partial w} \right\rVert_2 
= \frac{1}{\sqrt{n}} x \sqrt{\sum_{j=1}^n a_j^2 \phi'(w_jx)^2} 
= O(1).
\end{equation} 
The spectral norm of the Hessian at initialization is
\begin{equation}
    \left\lVert \frac{\partial^2 f_{\theta}(x) }{\partial w^2}  \right\rVert_{\op}
    = \frac{1}{\sqrt{n}} x^2 \max_{j=1,\ldots,n} |a_j \phi''(w_jx)|
    = O\left(\frac{\sqrt{\log n}}{\sqrt{n}}\right).
\end{equation}
This shows that the network is $O(1)$-Lipschitz in the parameters, while its parameter-Jacobian is $O((\log n)^c/\sqrt{n})$-Lipschitz. 


\subsection{Convergence of the neural network at initialization}
\citet{lee2018dnnsgps, matthews2018gaussian} prove:
\NNGPconvergence*
This directly implies that the network values are in a compact set at initialization:
\begin{lemma}
\label{lemma: initial function in compact set}
For any $\delta_0>0$, there is $K_0'>0$ such that for large enough layer width $n$, with probability $1-\delta_0$ over random initialization $\theta_0$: For any input $x\in M^d$:
\begin{equation}
    \left\lVert f_{\theta_0}(x) \right\rVert_2 \le K_0'.
\end{equation}
\end{lemma}

\subsection{Convergence of the empirical NTK at initialization}
The following result first appeared in \citet{jacot2018neural}, and was subsequently extended to more general architectures in \citet{yang2020tensorprograms2}. \citet{arora2019exact} show convergence uniformly over input vectors in a compact set, which we restate in the following.
\begin{lemma}
Consider random initialization $\theta_0$. Then, $\hat{\Theta}_{\theta_0}$ converges in probability (uniformly over inputs $x,x' \in M^d$) to a deterministic kernel:
\begin{equation}
    \hat{\Theta}_{\theta_0}(x,x') = J_{\theta_0}(x) J_{\theta_0}(x')^{\top} \overset{p}{\longrightarrow} \Theta(x,x') \in \mathbb{R}^{K\times K}.
\end{equation}
For multidimensional outputs, i.e. $K>1$, $\Theta(x,x')$ is a diagonal matrix. Thus $\Theta(\mathbf{x},\mathbf{x}') \in \mathbb{R}^{NK\times N'K}$ consists of $K$ blocks of dimension $N\times N'$.
\end{lemma}
\citet{carvalho2025positivity} show that the NTK is strictly positive definite for all non-polynomial activations, which applies to our setting. We will use this for zero parameter regularizer $\beta=0$. For positive parameter regularizer $\beta>0$ this is not required, which would allow extending our analysis to more general loss functionals that involve infinitely many training points (e.g. PINNs).
\section{Properties of the function-space loss}
\label{appendix: functionspaceloss}

In this section we introduce various assumptions on the function-space loss and discuss its properties. Table \ref{table: assumptions} provides an overview.

\begin{table}[hbt!]
\caption{Summary of assumptions on the function-space loss $\mathcal{C}:\mathbb{R}^{NK}\to\mathbb{R}$.}
\label{table: assumptions}
\centering
\small
\renewcommand{\arraystretch}{1.4}
\begin{tabular}{@{}clp{0.5\linewidth}p{0.2\linewidth}@{}}
\toprule
\# & Description & Summary \\
\midrule
(\ref{ass: Frechet bounded}) &
Bounded function-space gradient &
$\forall \mathbf{z}\in \mathcal{S}_0:
\quad \norm{\nabla_{\mathbf{z}}\mathcal{C}(\mathbf{z})}_2 \le K_1$. 
 \\

(\ref{assumption: upperboundedfunctionspacehessian}) &
Bounded gradient growth &
$\forall\mathbf{z}\in\mathcal{S}_0:
\quad \norm{\nabla_{\mathbf{z}}\mathcal{C}(\mathbf{z})}_{2}^2 \le 2K_2 \left(\mathcal{C}(\mathbf{z}) - \inf\mathcal{C}\right)$. 
 \\

(\ref{assumption: functionspaceplcondition}) &
Function-space PL &
$\forall \mathbf{z}\in\mathcal{S}_0:
\quad \norm{\nabla_{\mathbf{z}}\mathcal{C}(\mathbf{z})}_2^2 \ge 2\mu_{\mathcal{C}}\left(\mathcal{C}(\mathbf{z}) - \inf\mathcal{C}\right)$. 
 \\

\bottomrule
\end{tabular}
\end{table}

We write for $\mathbf{x}_1,\ldots,\mathbf{x}_N \in M^d$:
\begin{equation}
    \mathbf{f}_{\theta} 
    := f_{\theta}(\mathbf{x})
    := (f_{\theta}(\mathbf{x}_i))_i \in \mathbb{R}^{NK}
\end{equation}
and for the parameter-Jacobian and Hessian of the network,
\begin{equation}
    \mathbf{J}_{\theta} 
    := \frac{\partial}{\partial \theta} f_{\theta}(\mathbf{x}) \in \mathbb{R}^{NK\times p}
    \quad
    \text{and}
    \quad
    \mathbf{H}_{\theta} 
    := \frac{\partial^2}{\partial \theta^2} f_{\theta}(\mathbf{x}) \in \mathbb{R}^{NK\times p \times p}.
\end{equation}
The function-space loss gradient and Hessian are
\begin{equation}
    \nabla_{\mathbf{z}} \mathcal{C}(\mathbf{f}_{\theta})
    := \left(\nabla_{\mathbf{z}_i} \mathcal{C}(\mathbf{f}_{\theta}) \right)_{i} \in \mathbb{R}^{NK}
    \quad
    \text{and}
    \quad
    \nabla_{\mathbf{z} \mathbf{z}}^2 \mathcal{C}(\mathbf{f}_{\theta})
    := \left(\nabla^2_{\mathbf{z}_i \mathbf{z}_j} \mathcal{C}(\mathbf{f}_{\theta}) \right)_{ij} \in \mathbb{R}^{NK\times NK}.
\end{equation}
For losses summing over observations (such as MSE and categorical cross-entropy) the function-space Hessian will be a block-diagonal consisting of $N$ blocks of shape $K\times K$.
The chain rule gives the parameter-space loss gradient and Hessian
\begin{equation}
    \nabla_{\theta} \mathcal{C}(\mathbf{f}_{\theta}) 
    = \mathbf{J}_{\theta}^{\top} \nabla_{\mathbf{z}}\mathcal{C}(\mathbf{f}_{\theta}) 
    \in \mathbb{R}^p
    \quad
    \text{and}
    \quad
    \nabla_{\theta\theta}^2 \mathcal{C}(\mathbf{f}_{\theta}) 
    = \nabla_{\mathbf{z}}\mathcal{C}(\mathbf{f}_{\theta})^{\top} \mathbf{H}_{\theta}
    + \mathbf{J}_{\theta}^{\top} \nabla_{\mathbf{z}\mathbf{z}}^2\mathcal{C}(\mathbf{f}_{\theta}) \mathbf{J}_{\theta}
    \in \mathbb{R}^{p\times p}.
\end{equation}
For $\beta=0$ we assume that the function-space loss $\mathcal{C}$ is lower bounded, which implies that $\lim_{t\to T_{\max}} \mathcal{C}(\mathbf{f}_{\theta_t})$ exists and is finite.

For general $\beta\ge 0$, we have the following high-probability bound on the loss at initialization and during training:
\boundinitialloss*
\begin{proof}
Using the chain rule and the definition of the gradient flow, we have
\begin{equation}
    \frac{d}{dt} \mathcal{L}^{\beta}(\theta_t)
    = \nabla_{\theta}\mathcal{L}^{\beta}(\theta_t) ^{\top} \frac{d\theta_t}{dt}
    =-\eta_0 \nabla_{\theta}\mathcal{L}^{\beta}(\theta_t) ^{\top} \nabla_{\theta}\mathcal{L}^{\beta}(\theta_t)
    =-\eta_0 \left\lVert \nabla_{\theta}\mathcal{L}^{\beta}(\theta_t) \right\rVert_2^2 \le 0.
\end{equation}
Thus $\mathcal{L}^{\beta}(\theta_t) 
    \le \mathcal{L}^{\beta}(\theta_0) $ and therefore
\begin{equation}
    \mathcal{C}(\mathbf{f}_{\theta_t})
    \le \mathcal{C}(\mathbf{f}_{\theta_t}) + \frac{1}{2} \beta \left\lVert \theta_t - \theta_0 \right\rVert_2^2
    = \mathcal{L}^{\beta}(\theta_t) 
    \le \mathcal{L}^{\beta}(\theta_0) 
    = \mathcal{C}(\mathbf{f}_{\theta_0})
    \le K_0.
\end{equation}
The last inequality follows by Lemma \ref{lemma: initial function in compact set}.
\end{proof}
\begin{remark}
This Lemma does not apply to the weight-decay regularizer $\frac{1}{2}\left\lVert \theta \right\rVert_2^2$ analyzed in \citet{lewkowycz2020training}: In general, $\frac{1}{2}\left\lVert \theta_0 \right\rVert_2^2 \neq 0$, and thus for $\beta>0$ we cannot use $\mathcal{L}^{\beta}(\theta_0) = \mathcal{C}(f_{\theta_0})$. They show that using gradient flow with this regularizer leads to $f_{\theta_{\infty}}(\cdot)=0$ in the infinite-width limit, and we are thus not in the NTK regime.
\end{remark}

\subsection{Bounded function-space loss gradient}
\begin{assumption}[Bounded function-space loss gradient]
\label{ass: Frechet bounded}
There is $K_1>0$ such that the norm of the gradient of the function-space loss is bounded on the loss sublevel set:
\begin{equation}
    \left\lVert \nabla_{\mathbf{z}}\mathcal{C}(\mathbf{z}) \right\rVert_2 \le K_1
    \quad
    \text{for all $\mathbf{z}\in\mathcal{S}_0$.}
\end{equation}
\end{assumption}

\subsection{Bounded gradient growth}
\begin{assumption}[Bounded gradient growth]
\label{assumption: upperboundedfunctionspacehessian}
There is $K_2>0$ such that
\begin{equation}
    \lVert \nabla_{\mathbf{z}} \mathcal{C}(\mathbf{z}) \rVert_2^2
    \le 2 K_2 \left(\mathcal{C}(\mathbf{z}) - \inf \mathcal{C} \right)
    \quad
    \text{for all $\mathbf{z}\in\mathcal{S}_0$}.
\end{equation}
\end{assumption}
This assumption is satisfied if the eigenvalues of the function-space Hessian are globally bounded, as the following lemma shows.
\begin{lemma}
\label{lemma: functionspacebetasmoothness}
Assume that $\left\lVert \nabla_{\mathbf{z}\mathbf{z}}^2\mathcal{C}(\mathbf{z}) \right\rVert_{\op} \le K_2$ for all $\mathbf{z}\in\mathbb{R}^{NK}$. Then,
\begin{equation}
    \lVert \nabla_{\mathbf{z}} \mathcal{C}(\mathbf{z}) \rVert_2^2
    \le 2 K_2 \left(\mathcal{C}(\mathbf{z}) - \inf \mathcal{C} \right)
    \quad
    \text{for all $\mathbf{z}\in\mathbb{R}^{NK}$}.
\end{equation}
\end{lemma}
\begin{proof}
For $\mathbf{y},\mathbf{z}\in\mathbb{R}^{NK}$,
\begin{align}
    \inf C
    &\le \mathcal{C}(\mathbf{y}) \\
   & = \mathcal{C}(\mathbf{z}) 
    + \nabla_{\mathbf{z}} \mathcal{C}(\mathbf{z})^{\top} (\mathbf{y} - \mathbf{z} )
    + \int_0^1 (1-s) (\mathbf{y}-\mathbf{z})^{\top} \nabla^2_{\mathbf{z}\mathbf{z}} \mathcal{C}\left(\mathbf{z} +s(\mathbf{y}-\mathbf{z})\right) (\mathbf{y}-\mathbf{z}) ds \\
    &\le \mathcal{C}(\mathbf{z}) 
    + \nabla_{\mathbf{z}} \mathcal{C}(\mathbf{z})^{\top} (\mathbf{y} - \mathbf{z} ) 
    + \frac{K_2}{2} \lVert \mathbf{y} - \mathbf{z}  \rVert_2^2.
\end{align}
Thus using $\mathbf{y} = \mathbf{z} - \frac{1}{K_2} \nabla_{\mathbf{z}} \mathcal{C}(\mathbf{z})$,
\begin{equation}
    \inf \mathcal{C} 
    \le \mathcal{C}(\mathbf{z}) - \frac{1}{2K_2} \lVert \nabla_{\mathbf{z}} \mathcal{C}(\mathbf{z}) \rVert_2^2.
\end{equation}

\end{proof}

\subsection{Function-space PL-condition}
This section is only relevant for zero parameter regularizer $\beta=0$.

\begin{assumption}[Function-space PL]
\label{assumption: functionspaceplcondition}
There is $\mu_{\mathcal{C}}>0$, such that
\begin{equation}
    \left\lVert \nabla_{\mathbf{z}} \mathcal{C}(\mathbf{z}) \right\rVert_2^2
    \ge 2\mu_{\mathcal{C}} \left(\mathcal{C}(\mathbf{z}) - \inf \mathcal{C}  \right)
    \quad
    \text{for all $\mathbf{z}\in\mathcal{S}_0$.}
\end{equation}
\end{assumption}
\citet{oymak2019overparameterized, liu2022loss} instead defined the PL-condition in parameter-space for all $\theta$ in a ball around $\theta_0$, whose radius is required to be large enough depending on $\mu_{\mathcal{L}}$. Along the trajectory, this means that there is $\mu_{\mathcal{L}}>0$ such that 
\begin{equation}
    \left\lVert \nabla_{\theta} \mathcal{C}(\mathbf{f}_{\theta_t}) \right\rVert_2^2
    \ge 2\mu_{\mathcal{L}} \left(\mathcal{C}(\mathbf{f}_{\theta_t}) - \inf \mathcal{C}\right)
    \quad \text{for all $t\ge 0$}.
\end{equation}
As we can lower bound the parameter-space loss-gradient norm by
\begin{equation}
    \left\lVert \nabla_{\theta} \mathcal{C}(\mathbf{f}_{\theta_t}) \right\rVert_2^2
    = \left\lVert \mathbf{J}_{\theta_t}^{\top} \nabla_{\mathbf{z}} \mathcal{C}(\mathbf{f}_{\theta_t})  \right\rVert_2^2
    \ge \lambda_{\min}(\hat{\Theta}_{\theta_t}(\mathbf{x},\mathbf{x})) \left\lVert \nabla_{\mathbf{z}} \mathcal{C}(\mathbf{f}_{\theta_t}) \right\rVert_2^2,
\end{equation}
it is natural to only assume the $\mu_{\mathcal{C}}$-PL-condition in function-space along the training trajectory if one can additionally show that the empirical neural tangent kernel is sufficiently well-conditioned during training.

A sufficient and easy-to-verify condition that implies the function-space PL condition is strong convexity. It is however not a necessary condition, as e.g. the cross-entropy loss (without a reference class) satisfies the PL-condition but is not strongly convex. For completeness, we prove this in the following lemma, under a slightly more general condition than strong convexity.
\begin{lemma}
\label{lemma: SCimpliesPL}
Assume that $\mathcal{S}_0$ is convex, and there is $\mu_{\mathcal{C}}>0$ such that
\begin{equation}
    \lambda_{\min}( \nabla_{\mathbf{z}\mathbf{z}}^2 \mathcal{C}(\mathbf{z}) )
    \ge \mu_{\mathcal{C}} > 0
    \quad
    \text{for all $\mathbf{z}\in\mathcal{S}_0$}.
\end{equation}
Then it satisfies the function-space PL-condition (Assumption \ref{assumption: functionspaceplcondition}).
\end{lemma}
\begin{proof}
We can write for $\mathbf{y},\mathbf{z} \in \mathcal{S}_0$,
\begin{align}
    \mathcal{C}(\mathbf{y})
   & = \mathcal{C}(\mathbf{z}) 
    + \nabla_{\mathbf{z}} \mathcal{C}(\mathbf{z})^{\top} (\mathbf{y} - \mathbf{z} )
    + \int_0^1 (1-s) (\mathbf{y}-\mathbf{z})^{\top} \nabla^2_{\mathbf{z}\mathbf{z}} \mathcal{C}\left(\mathbf{z} +s(\mathbf{y}-\mathbf{z})\right) (\mathbf{y}-\mathbf{z}) ds \\
    &\ge \mathcal{C}(\mathbf{z}) 
    + \nabla_{\mathbf{z}} \mathcal{C}(\mathbf{z})^{\top} (\mathbf{y} - \mathbf{z} ) 
    + \frac{\mu_{\mathcal{C}}}{2} \lVert \mathbf{y} - \mathbf{z}  \rVert_2^2 \\
    &\ge \mathcal{C}(\mathbf{z}) - \frac{1}{2\mu_{\mathcal{C}}} \lVert \nabla_{\mathbf{z}} \mathcal{C}(\mathbf{z}) \rVert_2^2.
\end{align}
The function-space PL-condition follows by using  $\inf\mathcal{C} = \inf_{\mathbf{y}\in\mathcal{S}_0} \mathcal{C}(\mathbf{y})$ and taking $\mathbf{z} = \mathbf{f}_{\theta_t}$.
\end{proof}

\subsection{Positive (semi-)definite function-space Hessian}
For $\beta>0$, we need to assume that the function-space loss $\mathcal{C}:\mathbb{R}^{NK}\to\mathbb{R}$ is convex, as shown by the following example. Consider two classes and a single training point $(x,y)$ with target probability $(0.5,0.5)$. The regularized function-space objective using the Brier score with a reference class is
\begin{equation}
    \mathcal{C}_{\BS}^{\beta}(z) = \frac{1}{2} \left(\frac{e^z}{1+e^z} -y\right)^2 + \frac{1}{2}\frac{\beta}{\Theta_{x,x}} (z-z_0)^2.
\end{equation}
Take $\Theta_{x,x}=1$ and $y=0.5$. Then, for $\beta = 0.01$, and $z_0=5.5$, $\mathcal{C}_{\BS}^{\beta}$ has 3 stationary points (see Figure \ref{fig:cz-multiple-minima}): A local minimum at $z\approx 0.98$, a local maximum at $z\approx 2.47$, and a global minimum at $z\approx 5.24$.
\begin{figure}[hbt!]
    \centering
    \includegraphics[width=\linewidth]{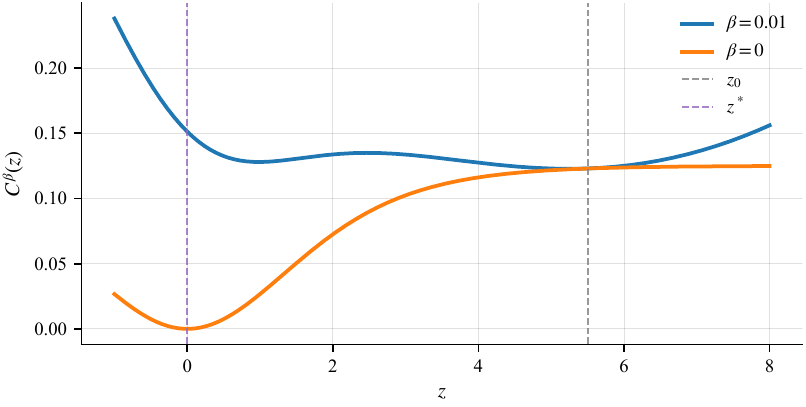}
    \caption{%
        The function-space Brier score with a regularizer can have multiple stationary points.
    }
\label{fig:cz-multiple-minima}
\end{figure}

This is fundamentally different from $\beta=0$, where the function-space loss is nonconvex but has a unique stationary point $z^* = \softmaxr^{-1}(0.5)=0$ (around which it is strictly convex).

For $\beta=0$, we do not need convexity to establish lazy training/exponential convergence, but can instead rely on the weaker function-space PL condition (Assumption \ref{assumption: functionspaceplcondition}). However, to show the closeness to the linearized training dynamics, we need to assume that the Hessian is strictly positive definite. We will relax this for the cross-entropy loss in Appendix \ref{appendix: extensionsoftmaxnoreference}.

\section{Examples of function-space losses}
\label{appendix: functionspacelossexamples}
\subsection{MSE loss}
Consider the function-space MSE loss
\begin{equation}
    \mathcal{C}_{\MSE}: \mathbb{R}^{NK} \to \mathbb{R},
    \quad
    \mathcal{C}_{\MSE}(\mathbf{z}) 
    :=\sum_{i=1}^N \frac{1}{2} \left\lVert \mathbf{z}_i - \mathbf{y}_i \right\rVert_2^2.
\end{equation}
It is twice continuously differentiable with
\begin{equation}
    \nabla_{\mathbf{z}}\mathcal{C}_{\MSE}(\mathbf{z}) = \mathbf{z} - \mathbf{y} \in \mathbb{R}^{NK}
    \quad 
    \text{and}
    \quad
    \nabla_{\mathbf{z}\mathbf{z}}^2 \mathcal{C}_{\MSE}(\mathbf{z}) = I_{NK}.
\end{equation}
\begin{lemma}
$\mathcal{C}_{\MSE}$ satisfies Assumption \ref{ass: Frechet bounded}.
\end{lemma}
\begin{proof}
For MSE loss we conveniently have
\begin{equation}
    \left\lVert \nabla_{\mathbf{z}}\mathcal{C}_{\MSE}(\mathbf{z}) \right\rVert_2^2
    = \sum_{i=1}^N \left\lVert \mathbf{z}_i - \mathbf{y}_i \right\rVert_2^2
    = 2\mathcal{C}_{\MSE}(\mathbf{z})
    \le 2K_0.
\end{equation}
\end{proof}
As $\nabla_{\mathbf{z}\mathbf{z}}^2\mathcal{C}_{\MSE}(\mathbf{z}) = I_{NK}$, the MSE loss satisfies Assumptions \ref{assumption: upperboundedfunctionspacehessian}, \ref{assumption: functionspaceplcondition} and its Hessian is strictly positive definite.

\subsection{Categorical cross-entropy loss}
\label{appx: cce}
Consider a classification task with $K$ classes. The probability simplex in $\mathbb{R}^K$ is
\begin{equation}
    \Delta^{K-1} 
    := \left\{p\in [0,1]^{K}: \sum_{k=1}^{K} p_k = 1 \right\}
    \subseteq \mathbb{R}^{K}.
\end{equation}
The $\mathrm{softmax}$-function is defined as
\begin{equation}
    \softmax: \mathbb{R}^K \to \interior(\Delta^{K-1}),
    \quad
    \softmax(z)_k
    := \frac{e^{z_k}}{\sum_{k'=1}^K e^{z_{k'}}}.
\end{equation}
For $N$ training points consider target probabilities $\mathbf{p}_i \in \Delta^{K-1}$. The function-space categorical cross-entropy loss is
\begin{equation}
    \mathcal{C}_{\CE}:\mathbb{R}^{NK} \to \mathbb{R},
    \quad
   \mathcal{C}_{\CE}(\mathbf{z})
   := - \sum_{i=1}^N \sum_{k=1}^{K} \mathbf{p}_{ik} \log(\softmax(\mathbf{z}_i)_k).
\end{equation}
This simplifies to
\begin{equation}
\mathcal{C}_{\CE}(\mathbf{z})
    = \sum_{i=1}^N \left( -\sum_{k=1}^{K} \mathbf{p}_{ik} \mathbf{z}_{ik} + \log\left(\sum_{k=1}^{K} e^{\mathbf{z}_{ik}} \right) \right).
\end{equation}
$\mathcal{C}_{\CE}$ is twice continuously differentiable with 
\begin{equation}
    \nabla_{\mathbf{z}_i} \mathcal{C}_{\CE}(\mathbf{z}) 
    = \softmax(\mathbf{z}_i) - \mathbf{p}_i
    \in \mathbb{R}^{K}, 
\end{equation}
\begin{equation}
    \nabla_{\mathbf{z}_i\mathbf{z}_i}^2 \mathcal{C}_{\CE}(\mathbf{z}) 
    = \diag\left(\softmax(\mathbf{z}_i) \right) - \softmax(\mathbf{z}_i)\softmax(\mathbf{z}_i)^{\top}
    \in \mathbb{R}^{K\times K}.
\end{equation}

\begin{lemma}
$\mathcal{C}_{\CE}$ satisfies Assumption \ref{ass: Frechet bounded}.
\end{lemma}
\begin{proof}
For any $\mathbf{z}\in\mathbb{R}^{NK}$,
\begin{equation}
    \left\lVert \nabla_{\mathbf{z}}\mathcal{C}_{\CE}(\mathbf{z}) \right\rVert_2^2 
    = \sum_{i=1}^N \left\lVert \nabla_{\mathbf{z}_i} \mathcal{C}_{\CE}(\mathbf{z}) \right\rVert_2^2
    = \sum_{i=1}^N \left\lVert \softmax(\mathbf{z}_i) -\mathbf{p}_i\right\rVert_2^2
    \le 2N.
\end{equation}
\end{proof}

\begin{lemma}
$\mathcal{C}_{\CE}$ satisfies Assumption \ref{assumption: upperboundedfunctionspacehessian}.
\end{lemma}
\begin{proof}
For any $\mathbf{z}\in\mathbb{R}^{NK}$ the function-space Hessian $\nabla_{\mathbf{z}\mathbf{z}}^2 \mathcal{C}_{\CE}(\mathbf{z}) \in \mathbb{R}^{NK\times NK}$ is a block-diagonal matrix with blocks $\nabla_{\mathbf{z}_i \mathbf{z}_i}^2 \mathcal{C}_{\CE}(\mathbf{z}) \in \mathbb{R}^{K\times K}$. Thus its spectral norm is
\begin{align}
    \left\lVert \nabla_{\mathbf{z}\mathbf{z}}^2 \mathcal{C}_{\CE}(\mathbf{z}) \right\rVert_{\op}
    &= \max_{i=1,\ldots,N} \left\lVert \nabla_{\mathbf{z}_i\mathbf{z}_i}^2 \mathcal{C}_{\CE}(\mathbf{z}) \right\rVert_{\op} \\
    &= \max_{i=1,\ldots,N} \left\lVert \diag(\softmax(\mathbf{z}_i)) - \softmax(\mathbf{z}_i)\softmax(\mathbf{z}_i)^{\top} \right\rVert_{\op}
    \le \frac{1}{2}.
\end{align}
\end{proof}

\begin{lemma}
Assume the targets are a discrete probability distribution with full support, i.e. $\mathbf{p}_{ik}\in (0,1)$. Then $\mathcal{C}_{\CE}$ satisfies Assumption \ref{assumption: functionspaceplcondition}.
\end{lemma}
\begin{proof}
Let $\mathbf{z} \in \mathcal{S}_0$.
As cross-entropy is nonnegative, we have for $i=1,\ldots,N$,
\begin{equation}
    -\sum_{k=1}^K \mathbf{p}_{ik} \log(\softmax(\mathbf{z}_i)_k)
    \le -\sum_{j=1}^N \sum_{k=1}^K \mathbf{p}_{jk} \log(\softmax(\mathbf{z}_j)_k)
    = \mathcal{C}_{\CE}(\mathbf{z}) 
    \le K_0.
\end{equation}
This gives a lower bound on the probits on the loss sublevel set:
\begin{equation}
    \softmax(\mathbf{z}_i)_k 
    \ge e^{-K_0/\mathbf{p}_{ik}} 
    \ge \min_{i,k} e^{-K_0/\mathbf{p}_{ik}} 
    =: 2\mu_{\mathcal{C}}>0.
\end{equation}
Using the $KL-\chi^2$ inequality it follows that
\begin{align}
    \mathcal{C}_{\CE}(\mathbf{z}) - \inf \mathcal{C}_{\CE}
    &=
    \mathcal{C}_{\CE}(\mathbf{z}) - \sum_{i=1}^N \Ent(\mathbf{p}_i) \\
    &= \sum_{i=1}^N\sum_{k=1}^K \mathbf{p}_{ik} \log\left(\frac{\mathbf{p}_{ik}}{\softmax(\mathbf{z}_i)_k}\right) \\
    &\le \sum_{i=1}^N\sum_{k=1}^K \frac{(\mathbf{p}_{ik} - \softmax(\mathbf{z}_i)_k)^2}{\softmax(\mathbf{z}_i)_k} \\
    &\le \frac{1}{2\mu_{\mathcal{C}}} \sum_{i=1}^N\sum_{k=1}^K (\mathbf{p}_{ik} - \softmax(\mathbf{z}_i)_k)^2 
    = \frac{1}{2\mu_{\mathcal{C}}} \left\lVert \nabla_{\mathbf{z}} \mathcal{C}_{\CE}(\mathbf{z}) \right\rVert_2^2.
\end{align}
\end{proof}
$\mathcal{C}_{\CE}$ is convex, but its Hessian has eigenvalue zero due to the redundancy $\softmax(z) = \softmax(z + \alpha 1)$.

\subsection{Categorical cross-entropy loss with a reference class}
\label{appx: cerc}
Consider $K+1$ classes with target probabilities $(\mathbf{p}_{i0}, \mathbf{p}_i) \in \Delta^K$.\footnote{We consider $K+1$ classes when using $\softmax$ with a reference class to achieve notational consistency with the function-space loss input dimension $NK$.} We use the $0$-th class with target probabilities $\mathbf{p}_{i0} = 1- \sum_{k=1}^K \mathbf{p}_{ik}$ as a reference class, by setting $z_0 = 0$.
We define the last $K$ coordinates of the $\mathrm{softmax}$-function with a reference class by
\begin{equation}
    \softmaxr: \mathbb{R}^K \to \mathbb{R}^K,
    \quad
    \softmaxr(z)_k 
    := 
    \frac{e^{z_k}}{1 + \sum_{k'=1}^K e^{z_{k'}}}.
\end{equation}
The probability of the reference class is $\softmaxr(z)_0 = (1 + \sum_{k=1}^K e^{z_{k}})^{-1}$.
The resulting $z \mapsto (\softmaxr(z)_0, \softmaxr(z))$ is a bijection between $K$-dimensional vectors and discrete probability distributions on $ \{0,\ldots,K \}$ with full support, i.e. $\interior(\Delta^{K})$.

The function-space categorical cross-entropy loss with a reference class is
\begin{equation}
    \mathcal{C}_{\CEr}:\mathbb{R}^{NK} \to \mathbb{R},
    \quad
   \mathcal{C}_{\CEr}(\mathbf{z})
   := - \sum_{i=1}^N \sum_{k=0}^{K} \mathbf{p}_{ik} \log(\softmaxr(\mathbf{z}_i)_k).
\end{equation}
This is equal to
\begin{equation}
\mathcal{C}_{\CEr}(\mathbf{z})
    = \sum_{i=1}^N \left( -\sum_{k=1}^{K} \mathbf{p}_{ik} \mathbf{z}_{ik} + \log\left(1 + \sum_{k=1}^{K} e^{\mathbf{z}_{ik}} \right) \right).
\end{equation}
For 2 classes (logit dimension $K=1$) we get the logit loss, which for one-hot labels $\mathbf{p}_{i1} \in \{0,1\}$ and $\mathbf{y}_i := 1_{\{\mathbf{p}_{i1}=1\}} - 1_{\{\mathbf{p}_{i1}=0\}}$ simplifies to
\begin{equation}
    \mathcal{C}_{CEr}(\mathbf{z}) 
    = \sum_{i=1}^N \left(- \mathbf{p}_{i1} \mathbf{z}_i + \log\left(1+e^{\mathbf{z}_i} \right) \right)
    = \sum_{i=1}^N \log\left(1 + e^{-\mathbf{y}_i\mathbf{z}_i} \right).
\end{equation}

$\mathcal{C}_{\CEr}$ is twice continuously differentiable with 
\begin{equation}
    \nabla_{\mathbf{z}_i} \mathcal{C}_{\CEr}(\mathbf{z}) 
    = \softmaxr(\mathbf{z}_i) - \mathbf{p}_i
    \in \mathbb{R}^{K}, 
\end{equation}
\begin{equation}
    \nabla_{\mathbf{z}_i\mathbf{z}_i}^2 \mathcal{C}_{\CEr}(\mathbf{z}) 
    = \diag\left(\softmaxr(\mathbf{z}_i) \right) - \softmaxr(\mathbf{z}_i)\softmaxr(\mathbf{z}_i)^{\top}
    \in \mathbb{R}^{K\times K}.
\end{equation}

We can show that the categorical cross-entropy loss with a reference class satisfies Assumptions \ref{ass: Frechet bounded}, \ref{assumption: upperboundedfunctionspacehessian}, \ref{assumption: functionspaceplcondition} just like for the standard categorical cross-entropy loss. Moreover, removing the redundancy along $\alpha\mapsto \mathbf{z}+\alpha 1$ makes its Hessian strictly positive definite.
\begin{lemma}
The Hessian of $\mathcal{C}_{\CEr}$ is strictly positive definite.
\end{lemma}
\begin{proof}
Write $\mathbf{s}_i := \softmaxr(\mathbf{z}_i)$, then $\nabla^2_{\mathbf{z}_i\mathbf{z}_i} \mathcal{C}_{\CEr}(\mathbf{z}) = \diag(\mathbf{s}_i) - \mathbf{s}_i\mathbf{s}_i^{\top}$. For any $v\in\mathbb{R}^{K}$ with $v\neq 0$ the Cauchy-Schwarz inequality gives
\begin{align}
    v^T \nabla^2_{\mathbf{z}_i\mathbf{z}_i} \mathcal{C}_{\CEr}(\mathbf{z}) v
    &= \sum_{k=1}^{K} \mathbf{s}_{ik} v_k^2 - \left(\sum_{k=1}^{K} \mathbf{s}_{ik} v_k\right)^2 \\
    &\ge \sum_{k=1}^{K} \mathbf{s}_{ik} v_k^2 - 
    \left(\sum_{k=1}^{K}\mathbf{s}_{ik} \right)\left(\sum_{k=1}^{K}\mathbf{s}_{ik}v_k^2 \right) \\ 
    &= \left(1- \sum_{k=1}^{K}\mathbf{s}_{ik}\right)\left(\sum_{k=1}^{K}\mathbf{s}_{ik}v_k^2 \right) 
    > 0.
\end{align}
\end{proof}

\section{Proofs for positive parameter regularizer}
The proof consists of two steps. First, we prove that the network is in the lazy training regime. Then, we prove that it can be approximated by its linearization during training.
\label{appendix: exponentialdecayregularizer}
\subsection{Exponential decay of the loss-gradient norm}
\begin{theorem}
Let $\beta>0$. Assume that Assumption \ref{ass: Frechet bounded} holds, and that the function-space loss is convex. Then there are constants $R$, $c_{\beta}>0$ such that for large enough layer width $n$: With probability $1-\delta_0$ over random initialization $\theta_0$, for all $t\ge 0$,
\begin{equation}
\label{eq: lazytraining}
    \left\lVert \frac{d}{dt}\theta_t \right\rVert_2
    \le \eta_0 c_{\beta}R e^{-c_{\beta}\eta_0t}
    \quad 
    \text{and thus} 
    \quad
    \left\lVert \theta_t - \theta_0 \right\rVert_2 \le R.
\end{equation}
\end{theorem}
\begin{proof}
By Assumption \ref{ass: Frechet bounded} there is $K_1 >0$, such that for $n$ large enough, with probability $1-\frac{1}{2}\delta_0$ over random initialization, $\left\lVert \nabla_{\mathbf{z}}\mathcal{C}(\mathbf{f}_{\theta_t}) \right\rVert_{2} \le K_1$. Further, using Lemma \ref{lemma: Jacobian Lipschitz Hessian bounded}, let $K'_1$ be the constant for local Lipschitzness/Boundedness of the Jacobian with probability $1-\frac{1}{2}\delta_0$ for $n$ large enough.

Define $c_{\beta} := \frac{1}{2}\beta$ and $R:= \frac{K_1 \sqrt{N} K'_1}{c_{\beta}}$. By Picard-Lindelöf the gradient flow ODE has a unique solution as long as $\theta_t \in B(\theta_0,R)$. Consider such $t\ge 0$. Recall that
\begin{equation}
    \frac{d\theta_t}{dt} 
    = -\eta_0 \nabla_{\theta} \mathcal{L}^{\beta}(\theta_t)
    = -\eta_0 \left( \nabla_{\theta} \mathcal{C}(\mathbf{f}_{\theta_t}) + \beta (\theta_t-\theta_0) \right).
\end{equation}
The dynamics of the squared norm of the gradient are
\begin{align}
    \frac{d}{dt} \left\lVert \nabla_{\theta} \mathcal{L}^{\beta}(\theta_t) \right\rVert_2^2
    &= 2 \left( \nabla_{\theta}\mathcal{L}^{\beta}(\theta_t) \right)^{\top} \nabla_{\theta\theta}^2 \mathcal{L}^{\beta}(\theta_t) \frac{d\theta_t}{dt} \\
    &= -2\eta_0 \left( \nabla_{\theta}\mathcal{L}^{\beta}(\theta_t) \right)^{\top} \nabla_{\theta\theta}^2 \mathcal{L}^{\beta}(\theta_t) \left( \nabla_{\theta} \mathcal{L}^{\beta}(\theta_t) \right).
\end{align}
The Hessian of the regularized loss is
\begin{equation}
    \nabla_{\theta\theta}^2 \mathcal{L}^{\beta}(\theta_t)
    = \nabla_{\theta\theta}^2 \mathcal{C}(\mathbf{f}_{\theta_t}) + \beta I_p
    = \nabla_{\mathbf{z}}\mathcal{C}(\mathbf{f}_{\theta_t})^{\top} \mathbf{H}_{\theta_t}
    + \mathbf{J}_{\theta_t}^{\top} \nabla^2_{\mathbf{z}}\mathcal{C}(\mathbf{f}_{\theta_t})\mathbf{J}_{\theta_t} 
    + \beta I_p
    \in \mathbb{R}^{p\times p}.
\end{equation}
We can bound the matrix norm of the model curvature term $\nabla_{\mathbf{z}}\mathcal{C}(\mathbf{f}_{\theta_t})^{\top} \mathbf{H}_{\theta_t}$ using Assumption \ref{ass: Frechet bounded} and the bound of the parameter-Hessian from Lemma \ref{lemma: Jacobian Lipschitz Hessian bounded}:
\begin{equation}
     \left\lVert \nabla_{\mathbf{z}}\mathcal{C}(\mathbf{f}_{\theta_t})^{\top} \mathbf{H}_{\theta_t} \right\rVert_{\op}
    \le \left\lVert \nabla_{\mathbf{z}}\mathcal{C}(\mathbf{f}_{\theta_t}) \right\rVert_{2}
    \left\lVert \mathbf{H}_{\theta_t} \right\rVert_{\op,2}
    \le K_1 \sqrt{N} \frac{(\log n)^c}{\sqrt{n}}K'_2 .
\end{equation}
This implies that its lowest negative eigenvalue is of order $O((\log n)^c/\sqrt{n})$.
The Gauss-Newton $\mathbf{J}_{\theta_t}^{\top} \nabla^2_{\mathbf{z}}\mathcal{C}(\mathbf{f}_{\theta_t})\mathbf{J}_{\theta_t} \in \mathbb{R}^{p\times p}$ is positive semi-definite. Hence for positive $\beta>0$ there is large enough layer width $n$ such that
\begin{equation}
    \lambda_{\min}\left( \nabla_{\theta\theta}^2 \mathcal{L}^{\beta} (\theta_t)\right)
     \ge -\frac{(\log n)^c}{\sqrt{n}} K_1 \sqrt{N} K'_2 + 0+\beta \ge \frac{\beta}{2} =: c_{\beta}.
\end{equation}
Thus,
\begin{equation}
    \frac{d}{dt} \left\lVert \nabla_{\theta} \mathcal{L}^{\beta}(\theta_t) \right\rVert_2^2
    \le -2\eta_0 c_{\beta} \left\lVert  \nabla_{\theta}\mathcal{L}^{\beta}(\theta_t)\right\rVert_2^2.
\end{equation}
By Gronwall's inequality, the norm of the gradient decays exponentially:
\begin{equation}
    \left\lVert \nabla_{\theta} \mathcal{L}^{\beta}(\theta_t)\right\rVert_2^2
    \le e^{-2\eta_0 c_{\beta}t} \left\lVert \nabla_{\theta} \mathcal{L}^{\beta}(\theta_0) \right\rVert_2^2.
\end{equation}
We can bound the norm of the initial gradient by using Assumption \ref{ass: Frechet bounded} and the bound on the parameter-Jacobian from Lemma \ref{lemma: Jacobian Lipschitz Hessian bounded}:
\begin{equation}
    \left\lVert \nabla_{\theta} \mathcal{L}^{\beta}(\theta_0) \right\rVert_2
    = \left\lVert \mathbf{J}_{\theta_0}^{\top}\nabla_{\mathbf{z}}\mathcal{C}(\mathbf{f}_{\theta_0}) \right\rVert_2
    \le \left\lVert  \nabla_{\mathbf{z}}\mathcal{C}(\mathbf{f}_{\theta_0}) \right\rVert_{2} \left\lVert \mathbf{J}_{\theta_0} \right\rVert_{2,2}
    \le K_1 \sqrt{N} K'_1.
\end{equation}
Hence,
\begin{equation}
    \left\lVert \frac{d\theta_t}{dt} \right\rVert_2
    = \eta_0 \left\lVert \nabla_{\theta} \mathcal{L}^{\beta}(\theta_t) \right\rVert_2
    \le \eta_0 e^{-\eta_0 c_{\beta}t} \left\lVert \nabla_{\theta} \mathcal{L}^{\beta}(\theta_0) \right\rVert_2
    \le \eta_0 K_1 \sqrt{N} K'_1  e^{-\eta_0 c_{\beta}t}.
\end{equation}
For the distance of parameters from initialization,
\begin{equation}
    \left\lVert \theta_t - \theta_0 \right\rVert_2 
    \le \int_0^t \left\lVert \frac{d\theta_t}{du} \right\rVert_2 du
    \le \eta_0 K_1 \sqrt{N} K'_1 \int_0^t  e^{-\eta_0 c_{\beta}u} du
    = \frac{K_1 \sqrt{N} K'_1}{c_{\beta}} (1 - e^{-\eta_0 c_{\beta}t}).
\end{equation}
Thus $\left\lVert \theta_t - \theta_0 \right\rVert_2 < R$, and the above holds for all $t\ge 0$.
\end{proof}
\subsection{Closeness to linearization}
\begin{theorem}
\label{thm: closenesslinpositivebeta}
Let $\beta>0$. Consider a function-space loss that satisfies Assumption \ref{ass: Frechet bounded} and is convex.
Then there is a constant $C_2>0$ such that for large enough layer width $n$: With probability $1-\delta_0$ over random initialization $\theta_0$,
\begin{equation}
    \forall x\in M^d, t\ge 0: 
    \quad
    \left\lVert f_{\theta_t}(x) - f_{\theta_t^{\lin}}^{\lin}(x)\right\rVert_2 
    \le C_2 \frac{(\log n)^c}{\sqrt{n}}.
\end{equation}
\end{theorem}
Recall that the linearized network and parameters of training with the linearized network are
\begin{equation}
    f^{\lin}_{\theta}(x) = f_{\theta_0}(x) + J_{\theta_0}(x) (\theta - \theta_0)
    \quad
    \text{and} 
    \quad
    \frac{d}{dt}\theta_t^{\lin} = -\eta_0 \left(\nabla_{\theta}\mathcal{C}(\mathbf{f}^{\lin}_{\theta_t^{\lin}}) + \beta (\theta_t^{\lin} - \theta_0)\right).
\end{equation}
For general function-space losses, $\theta_t^{\lin}$ does not follow a linear ODE. However, its dynamics are much easier to understand.
Using the triangle inequality, we bound
\begin{align}
    \left\lVert f_{\theta_t}(x) - f^{\lin}_{\theta_t^{\lin}}(x) \right\rVert_2
    &\le \left\lVert f_{\theta_t}(x) - f^{\lin}_{\theta_t}(x) \right\rVert_2 
    + \left\lVert f_{\theta_t}^{\lin}(x) - f^{\lin}_{\theta_t^{\lin}}(x) \right\rVert_2 \\
    &= \left\lVert f_{\theta_t}(x) - f^{\lin}_{\theta_t}(x) \right\rVert_2 
    + \left\lVert J_{\theta_0}(x) \left( \theta_t - \theta_t^{\lin} \right) \right\rVert_2
\end{align}
We will show that $\left\lVert f_{\theta_t}(x) - f^{\lin}_{\theta_t}(x) \right\rVert_2$ and $\left\lVert \theta_t - \theta_t^{\lin} \right\rVert_2$ are $O((\log n)^c/\sqrt{n})$.

\begin{lemma}
\label{lemma: trajectoriescompactset}
Let $\delta>0$. The lazy training property (\ref{eq: lazytraining}) implies that there is large enough layer width $n$ such that with probability $1-\delta_0$ over random initialization $\theta_0$, the trajectories of the training outputs of the network and its linearization are in a compact set: 
\begin{equation}
    \mathbf{f}_{\theta_t}, \mathbf{f}^{\lin}_{\theta_t}, \mathbf{f}_{\theta^{\lin}_t}, \mathbf{f}^{\lin}_{\theta^{\lin}_t}
    \in B(0; \sqrt{N}(K_1' R + K_0'))
    \quad \text{for all $t\ge 0$.}
\end{equation}
\end{lemma}
Denote by $L_{\mathcal{C}}$ and $L_{\nabla\mathcal{C}}$ the Lipschitz-constants of $\mathcal{C}$ and $\nabla \mathcal{C}$ on this compact set.
\begin{proof}
Using Lemma \ref{lemma: boundinitialloss} and the bound on the parameter-Jacobian from Lemma \ref{lemma: Jacobian Lipschitz Hessian bounded}, as well as the bound on the initial function values from Lemma \ref{lemma: initial function in compact set}:
\begin{align}
    \lVert \mathbf{f}_{\theta_t} \rVert_2
    &\le \lVert \mathbf{f}_{\theta_t} - \mathbf{f}_{\theta_0} \rVert_2
    + \lVert \mathbf{f}_{\theta_0} \rVert_2  \\
    &\le \sup_{u\in [0,1]} \lVert \mathbf{J}_{\theta_0 + u(\theta_t-\theta_0)} \rVert_{2,2} \lVert \theta_t - \theta_0 \rVert_2 + \sqrt{N}K_0' \\
    &\le \sqrt{N}K_1' R + \sqrt{N}K_0'.
\end{align}
\end{proof}

\begin{lemma}
\label{lemma: fflindistancesametheta}
Let $\beta \ge 0$. Assume that equation (\ref{eq: lazytraining}) holds (lazy training). Then,
\begin{equation}
\forall x\in M^d,t\ge 0:
\quad
    \left\lVert f_{\theta_t}(x) - f^{\lin}_{\theta_t}(x) \right\rVert_2
    \le \frac{(\log n)^c}{\sqrt{n}}K'_2 R^2 .
\end{equation}
\end{lemma}
\begin{proof}
To bound $\left\lVert f_{\theta_t}(x) - f^{\lin}_{\theta_t}(x) \right\rVert_2$, we use Lemma \ref{lemma: Jacobian Lipschitz Hessian bounded} and equation (\ref{eq: lazytraining}).
\begin{align}
    \left\lVert \frac{d}{dt} \left( f_{\theta_t}(x) - f^{\lin}_{\theta_t}(x) \right)\right\rVert_2
    &= \left\lVert \left( J_{\theta_t}(x) - J_{\theta_0}(x) \right)\frac{d\theta_t}{dt}  \right\rVert_2 \\
    &\le \left\lVert J_{\theta_t}(x) - J_{\theta_0}(x) \right\rVert_2 \left\lVert\frac{d\theta_t}{dt} \right\rVert_2
    \le \frac{(\log n)^c}{\sqrt{n}}K'_2R\eta_0c_{\beta}R e^{-c_{\beta}\eta_0 t}.
\end{align}
By integrating it follows that
\begin{equation}
    \left\lVert f_{\theta_t}(x) - f^{\lin}_{\theta_t}(x) \right\rVert_2
    \le \frac{(\log n)^c}{\sqrt{n}} K'_2R^2\eta_0c_{\beta} \int_0^t e^{-c_{\beta}\eta_0 u}du
    \le \frac{(\log n)^c}{\sqrt{n}}K'_2 R^2 .
\end{equation}
\end{proof}
\begin{lemma}
\label{lemma: dthetaresgn}
Let $\beta \ge 0$. Assume that equation (\ref{eq: lazytraining}) holds (lazy training). We can write
\begin{equation}
    \frac{d}{dt} (\theta_t - \theta_t^{\lin})
    = -\eta_0 \Delta_t 
    - \eta_0 \left( \mathbf{J}_{\theta_0}^{\top} \overline{\mathbf{H}}_t \mathbf{J}_{\theta_0} + \beta I_p \right)(\theta_t - \theta_t^{\lin}),
\end{equation}
where we define the averaged function-space Gauss Newton
\begin{equation}
    \overline{\mathbf{H}}_t 
    := \int_0^1  \nabla_{\mathbf{z}\mathbf{z}}^2\mathcal{C}\left(\mathbf{f}_{\theta_t^{\lin}}^{\lin} + s (\mathbf{f}_{\theta_t}^{\lin} - \mathbf{f}_{\theta_t^{\lin}}^{\lin})\right) ds \in \mathbb{R}^{NK\times NK}.
\end{equation}
Moreover, the residual term $\Delta_t\in\mathbb{R}^p$ satisfies (for some $K^{\Delta}>0$) 
\begin{equation}
    \left\lVert \Delta_t \right\rVert_2 \le \frac{(\log n)^c}{\sqrt{n}} K^{\Delta}
    \quad
    \text{for all $t\ge 0$.}
\end{equation}
\end{lemma}
\begin{proof}
We can write
\begin{align}
    \frac{d}{dt} \left(\theta_t - \theta_t^{\lin} \right)
    =& -\eta_0 \left(\mathbf{J}_{\theta_t}^{\top}\nabla_{\mathbf{z}}\mathcal{C}(\mathbf{f}_{\theta_t}) - \mathbf{J}_{\theta_0}^{\top}\nabla_{\mathbf{z}}\mathcal{C}\left(\mathbf{f}^{\lin}_{\theta_t^{\lin}}\right)  \right) \\
    =& -\eta_0 \left( \mathbf{J}_{\theta_t} - \mathbf{J}_{\theta_0} \right)^{\top} \nabla_{\mathbf{z}}\mathcal{C}(\mathbf{f}_{\theta_t})  \label{lincloseterm1} \\
    &-\eta_0 \mathbf{J}_{\theta_0}^{\top}\left(\nabla_{\mathbf{z}}\mathcal{C}(\mathbf{f}_{\theta_t}) - \nabla_{\mathbf{z}}\mathcal{C}\left(\mathbf{f}^{\lin}_{\theta_t}\right)\right) 
    \label{lincloseterm2}
    \\
    & - \eta_0 \mathbf{J}_{\theta_0}^{\top}\left( \nabla_{\mathbf{z}} \mathcal{C}\left(\mathbf{f}^{\lin}_{\theta_t}\right) - \nabla_{\mathbf{z}} \mathcal{C}\left(\mathbf{f}^{\lin}_{\theta_t^{\lin}}\right) \right)  \label{lincloseterm3} \\
    & - \beta (\theta_t - \theta_t^{\lin}).
\end{align}
We summarize the first two terms (\ref{lincloseterm1}) and (\ref{lincloseterm2}) as
\begin{equation}
\Delta_t 
:= \left(\mathbf{J}_{\theta_t} - \mathbf{J}_{\theta_0} \right)^{\top} \nabla_{\mathbf{z}}\mathcal{C}(\mathbf{f}_{\theta_t}) 
+ \mathbf{J}_{\theta_0}^{\top}\left(\nabla_{\mathbf{z}}\mathcal{C}(\mathbf{f}_{\theta_t}) - \nabla_{\mathbf{z}}\mathcal{C}(\mathbf{f}_{\theta_t}^{\lin})\right) \in \mathbb{R}^{p}.
\end{equation}
Using the Fundamental Theorem of Calculus, we write the third term (\ref{lincloseterm3}) as
\begin{align}
    \mathbf{J}_{\theta_0}^{\top}\left( \nabla_{\mathbf{z}} \mathcal{C}(\mathbf{f}_{\theta_t}^{\lin}) - \nabla_{\mathbf{z}} \mathcal{C}\left(\mathbf{f}_{\theta_t^{\lin}}^{\lin}\right) \right)
    &= \mathbf{J}_{\theta_0}^{\top} \int_0^1 \nabla_{\mathbf{z}\mathbf{z}}^2 \mathcal{C}\left(\mathbf{f}_{\theta_t^{\lin}}^{\lin} + s (\mathbf{f}_{\theta_t}^{\lin} - \mathbf{f}_{\theta_t^{\lin}}^{\lin})\right)\left(\mathbf{f}_{\theta_t}^{\lin} - \mathbf{f}_{\theta_t^{\lin}}^{\lin}\right) ds \\
    &= \mathbf{J}_{\theta_0}^{\top} \int_0^1 \nabla_{\mathbf{z}\mathbf{z}}^2 \mathcal{C}\left(\mathbf{f}_{\theta_t^{\lin}}^{\lin} + s (\mathbf{f}_{\theta_t}^{\lin} - \mathbf{f}_{\theta_t^{\lin}}^{\lin})\right) ds \mathbf{J}_{\theta_0}(\theta_t - \theta_t^{\lin}) \\
    &= \mathbf{J}_{\theta_0}^{\top} \overline{\mathbf{H}}_t \mathbf{J}_{\theta_0}(\theta_t - \theta_t^{\lin}).
\end{align}
To bound the first term (\ref{lincloseterm1}) of $\Delta_t$, we use Lemma \ref{lemma: trajectoriescompactset} and Lemma \ref{lemma: Jacobian Lipschitz Hessian bounded}:
\begin{equation}
    \left\lVert \left(\mathbf{J}_{\theta_t} - \mathbf{J}_{\theta_0} \right)^{\top} \nabla_{\mathbf{z}}\mathcal{C}(\mathbf{f}_{\theta_t})  \right\rVert_2
    \le \left\lVert \nabla_{\mathbf{z}} \mathcal{C}(\mathbf{f}_{\theta_t}) \right\rVert_{2} \left\lVert \mathbf{J}_{\theta_t} - \mathbf{J}_{\theta_0} \right\rVert_{2,2}
    \le L_{\mathcal{C}} \sqrt{N} \frac{(\log n)^c}{\sqrt{n}} K'_2R.
\end{equation}
To bound the second term (\ref{lincloseterm2}) of $\Delta_t$, we use Lemma \ref{lemma: trajectoriescompactset}, Lemma \ref{lemma: fflindistancesametheta}, and Lemma \ref{lemma: Jacobian Lipschitz Hessian bounded}:
\begin{align}
    \left\lVert \mathbf{J}_{\theta_0}^{\top} \left( \nabla_{\mathbf{z}}\mathcal{C}(\mathbf{f}_{\theta_t}) - \nabla_{\mathbf{z}}\mathcal{C}(\mathbf{f}_{\theta_t}^{\lin})\right) \right\rVert_2
    &\le \left\lVert \nabla_{\mathbf{z}}\mathcal{C}(\mathbf{f}_{\theta_t}) - \nabla_{\mathbf{z}}\mathcal{C}(\mathbf{f}_{\theta_t}^{\lin}) \right\rVert_{2} \left\lVert \mathbf{J}_{\theta_0}\right\rVert_{2,2} \\
    &\le L_{\nabla \mathcal{C}} \left\lVert \mathbf{f}_{\theta_t} - \mathbf{f}_{\theta_t}^{\lin} \right\rVert_{2} \sqrt{N} K'_1
    \le \frac{(\log n)^c}{\sqrt{n}}L_{\nabla \mathcal{C}} K'_1 K'_2 N R^2.
\end{align}
Defining $K^{\Delta} := L_{\mathcal{C}} \sqrt{N} K'_2R + L_{\nabla \mathcal{C}} K'_1 K'_2 N R^2$ we get the $O((\log n)^c/\sqrt{n})$-bound on $\sup_{t\ge 0} \lVert \Delta_t \rVert_2$.
\end{proof}

\begin{proof}[Proof of Theorem \ref{thm: closenesslinpositivebeta}]
Write
\begin{equation}
    \left\lVert f_{\theta_t}(x) - f^{\lin}_{\theta_t^{\lin}}(x) \right\rVert_2
    \le \left\lVert f_{\theta_t}(x) - f^{\lin}_{\theta_t}(x) \right\rVert_2 + \left\lVert J_{\theta_0}(x) \right\rVert_2 \left\lVert \theta_t - \theta_t^{\lin} \right\rVert_2.
\end{equation}
Lemma \ref{lemma: fflindistancesametheta} gives
\begin{equation}
    \forall x\in M^d,t\ge 0:
    \quad
    \left\lVert f_{\theta_t}(x) - f^{\lin}_{\theta_t}(x) \right\rVert_2
    \le \frac{(\log n)^c}{\sqrt{n}} K_2' R^2.
\end{equation}
Using Lemma \ref{lemma: dthetaresgn} and $\lambda_{\min}\left( \mathbf{J}_{\theta_0}^{\top} \overline{\mathbf{H}}_t \mathbf{J}_{\theta_0} + \beta I_p\right) \ge \beta$,
\begin{align}
    \frac{d}{dt} \frac{1}{2} \lVert \theta_t - \theta_t^{\lin} \rVert_2^2
    &= (\theta_t - \theta_t^{\lin})^{\top} \frac{d}{dt} (\theta_t - \theta_t^{\lin}) \\
    &= -\eta_0 (\theta_t - \theta_t^{\lin})^{\top} \Delta_t 
    - \eta_0 (\theta_t - \theta_t^{\lin})^{\top} \left(\mathbf{J}_{\theta_0}^{\top} \overline{\mathbf{H}}_t \mathbf{J}_{\theta_0} + \beta I_p \right)(\theta_t - \theta_t^{\lin}) \\
    &\le \eta_0 \lVert \theta_t - \theta_t^{\lin} \rVert_2 K^{\Delta} \frac{(\log n)^c}{\sqrt{n}}
    - \eta_0 \beta \lVert \theta_t - \theta_t^{\lin} \rVert_2^2.
\end{align}
As we also have $\frac{d}{dt} \frac{1}{2} \left\lVert \theta_t - \theta_t^{\lin} \right\rVert_2^2 = \left\lVert \theta_t - \theta_t^{\lin} \right\rVert_2 \frac{d}{dt} \left\lVert \theta_t - \theta_t^{\lin} \right\rVert_2 $, we get by Gronwall's inequality,
\begin{equation}
    \lVert \theta_t - \theta_t^{\lin} \rVert_2
    \le \frac{(\log n)^c}{\sqrt{n}} \frac{K^{\Delta}}{\beta} \left(1 - e^{-\beta \eta_0 t} \right)
    \le \frac{(\log n)^c}{\sqrt{n}} \frac{K^{\Delta}}{\beta}.
\end{equation}
This concludes the proof with
\begin{equation}
    C_2 
    := K_2'R^2 + K_1'\frac{K^{\Delta}}{\beta}.
\end{equation}
\end{proof}

\section{Proofs for zero parameter regularizer}
\label{appendix: exponentialdecayPL}

\subsection{Exponential decay of the loss-gradient norm}

\begin{theorem}
\label{thm: expdecaybetazero}
Let $\beta=0$. Consider a function-space loss that satisfies Assumptions \ref{assumption: upperboundedfunctionspacehessian} and \ref{assumption: functionspaceplcondition}. Then there are $R,c_0>0$ such that for large enough layer width $n$: 
With probability $1-\delta_0$ over random initialization $\theta_0$, for all $t\ge 0$,
\begin{equation}
    \left\lVert \frac{d}{dt}\theta_t \right\rVert_2
    \le \eta_0 c_{0}R e^{-c_{0}\eta_0t}
    \quad 
    \text{and thus} 
    \quad
    \left\lVert \theta_t - \theta_0 \right\rVert_2 \le R.
\end{equation}
\end{theorem}
The proof closely follows \citet{oymak2019overparameterized, liu2022loss}, where it was presented for discrete-time gradient descent.

\begin{proof}
Recall that by Lemma \ref{lemma: boundinitialloss} with high probability over random initialization $\theta_0$,
\begin{equation}
    \mathcal{C}(\mathbf{f}_{\theta_0}) \le K_0.
\end{equation}

Define $c_0 := \mu_{\mathcal{C}} \lambda_{\min}(\Theta_{\mathbf{x},\mathbf{x}})/2$ and $R := \frac{\sqrt{2K_2 N K_1'^2(K_0 - \inf\mathcal{C})}}{c_0}$, using the constants $K_2, \mu_{\mathcal{C}}$ from Assumptions \ref{assumption: upperboundedfunctionspacehessian} and \ref{assumption: functionspaceplcondition}.
Consider $t\ge 0$ as long as $\left\lVert \theta_t - \theta_0 \right\rVert_2 \le R$.
Due to the convergence of the NTK at initialization, this implies that it is still sufficiently well conditioned at time $t$, i.e. there is $n$ large enough (only depending on $R$), such that $\lambda_{\min}(\hat{\Theta}_{\theta_t}(\mathbf{x},\mathbf{x})) \ge \frac{1}{2}\lambda_{\min}(\Theta_{\mathbf{x},\mathbf{x}})$.
Now, Assumption \ref{assumption: functionspaceplcondition} implies that the loss gap decays exponentially quick during training:
\begin{align}
    \frac{d}{dt}\mathcal{C}(\mathbf{f}_{\theta_t})
    = \nabla_{\theta}\mathcal{C}(\mathbf{f}_{\theta_t}) \frac{d}{dt}\theta_t
    &= - \eta_0 \left\lVert \nabla_{\theta}\mathcal{C}(\mathbf{f}_{\theta_t}) \right\rVert_2^2 \\
    &\le -\eta_0 \lambda_{\min}(\hat{\Theta}_{\theta_t}(\mathbf{x},\mathbf{x})) \left\lVert \nabla_{\mathbf{z}} \mathcal{C}(\mathbf{f}_{\theta_t}) \right\rVert_2^2  \\
    &\le -\eta_0 \frac{1}{2}\lambda_{\min}(\Theta_{\mathbf{x},\mathbf{x}}) 2\mu_{\mathcal{C}} \left(\mathcal{C}(\mathbf{f}_{\theta_t}) - \inf\mathcal{C} \right),
\end{align}
and thus Gronwall's inequality,
\begin{equation}
    \mathcal{C}(\mathbf{f}_{\theta_t}) - \inf\mathcal{C}
    \le \left(\mathcal{C}(\mathbf{f}_{\theta_0}) - \inf\mathcal{C} \right) e^{-2c_0\eta_0 t}
    \le (K_0 - \inf\mathcal{C}) e^{-2c_0\eta_0 t}. 
\end{equation}

Using Assumption \ref{assumption: upperboundedfunctionspacehessian}, one can now use the exponential decay of the loss gap, to show that the norm of the loss-gradient decays exponentially quick:
\begin{align}
    \left\lVert \frac{d}{dt}\theta_t\right\rVert_2
    = \eta_0 \left\lVert \nabla_{\theta} \mathcal{C}(\mathbf{f}_{\theta_t}) \right\rVert_2
    &\le \eta_0 \sqrt{N}K_1' \left\lVert \nabla_{\mathbf{z}} \mathcal{C}(\mathbf{f}_{\theta_t}) \right\rVert_2 \\
    &\le \eta_0 \sqrt{N} K_1' 
    \sqrt{2K_2\left(\mathcal{C}(\mathbf{f}_{\theta_t}) - \inf\mathcal{C} \right)} \\
    &\le \eta_0 \sqrt{2NK_1'^2K_2 (K_0 - \inf\mathcal{C})} e^{-c_0\eta_0 t}.
\end{align}
This also implies that
\begin{equation}
    \left\lVert \theta_t - \theta_0 \right\rVert_2 
    \le \int_0^t \left\lVert \frac{d}{du}\theta_u \right\rVert_2 du
    \le \eta_0 \sqrt{2NK_1'^2K_2 (K_0 - \inf\mathcal{C})} \int_0^t e^{-c_0\eta_0 u}du
    < R.
\end{equation}
Thus the above holds for all $t\ge 0$.
\end{proof}

\subsection{Closeness to linearization}
\begin{theorem}
\label{thm: closenesslinbetazero}
Let $\beta=0$. Consider a function-space loss that satisfies Assumptions \ref{assumption: upperboundedfunctionspacehessian}, \ref{assumption: functionspaceplcondition}, and has strictly positive definite Hessian. Then there is a constant $C_2>0$ such that for large enough layer width $n$: With probability $1-\delta_0$ over random initialization $\theta_0$,
\begin{equation}
    \forall x\in M^d, t\ge 0: 
    \quad
    \left\lVert f_{\theta_t}(x) - f_{\theta_t^{\lin}}^{\lin}(x)\right\rVert_2 
    \le C_2 \frac{(\log n)^c}{\sqrt{n}}.
\end{equation}
\end{theorem}

\begin{proof}
We will follow the proof of Theorem \ref{thm: closenesslinpositivebeta} in function-space instead of parameter-space, using Lemma \ref{lemma: fflindistancesametheta} and Lemma \ref{lemma: dthetaresgn} for $\beta=0$.
Write
\begin{equation}
    \left\lVert f_{\theta_t}(x) - f^{\lin}_{\theta_t^{\lin}}(x) \right\rVert_2
    \le \left\lVert f_{\theta_t}(x) - f^{\lin}_{\theta_t}(x) \right\rVert_2 + \left\lVert f_{\theta_t}^{\lin}(x) - f^{\lin}_{\theta_t^{\lin}}(x) \right\rVert_2.
\end{equation}
Lemma \ref{lemma: fflindistancesametheta} gives
\begin{equation}
    \forall x\in M^d,t\ge 0:
    \quad
    \left\lVert f_{\theta_t}(x) - f^{\lin}_{\theta_t}(x) \right\rVert_2
    \le \frac{(\log n)^c}{\sqrt{n}} K_2' R^2.
\end{equation}
Using Lemma \ref{lemma: dthetaresgn} gives (after multiplying both sides with $J_{\theta_0}(x)$ to get into function-space)
\begin{equation}
\label{eq: functionspacedeltasplitnormal}
    \frac{d}{dt} \left( f_{\theta_t}^{\lin}(x) - f_{\theta_t^{\lin}}^{\lin}(x)  \right)
    = -\eta_0 J_{\theta_0}(x)\Delta_t
    - \eta_0 J_{\theta_0}(x)\mathbf{J}_{\theta_0}^{\top} \overline{\mathbf{H}}_t (\mathbf{f}_{\theta_t}^{\lin} - \mathbf{f}_{\theta_t^{\lin}}^{\lin}).
\end{equation}
In particular, in the training points:
\begin{equation}
    \frac{d}{dt} (\mathbf{f}_{\theta_t}^{\lin} - \mathbf{f}_{\theta_t^{\lin}}^{\lin})
    = -\eta_0 \mathbf{J}_{\theta_0} \Delta_t
    -\eta_0 \hat{\mathbf{\Theta}}_0 \overline{\mathbf{H}}_t (\mathbf{f}_{\theta_t}^{\lin} - \mathbf{f}_{\theta_t^{\lin}}^{\lin}).
\end{equation}

We will now give a lower bound on the lowest eigenvalue of $\overline{\mathbf{H}}_t$, to conclude with a bound on $\left\lVert \mathbf{f}_{\theta_t}^{\lin} - \mathbf{f}^{\lin}_{\theta_t^{\lin}} \right\rVert_2$ via Gronwall's inequality. By Lemma \ref{lemma: trajectoriescompactset}, the values $\mathbf{f}_{\theta_t^{\lin}}^{\lin} + s \left(\mathbf{f}_{\theta_t}^{\lin} - \mathbf{f}_{\theta_t^{\lin}}^{\lin}\right)$ are in the compact set $B(0;\sqrt{N} (K_0' + K_1'R))$.
Define
\begin{equation}
    a_{\mathcal{C}} := \min_{\mathbf{z} \in B(0;\sqrt{N} (K_0' + K_1'R))} \lambda_{\min}(\nabla_{\mathbf{z}\mathbf{z}}^2 \mathcal{C}(\mathbf{z})) > 0.
\end{equation}
Then for any $\mathbf{z}\in\mathbb{R}^{NK}$ and $t\ge 0$,
\begin{equation}
    \mathbf{z}^{\top} \overline{\mathbf{H}}_t \mathbf{z}
    = \int_0^1 \mathbf{z}^{\top}\nabla_{\mathbf{z}\mathbf{z}}^2 \mathcal{C}\left(\mathbf{f}_{\theta_t^{\lin}}^{\lin} + s \left(\mathbf{f}_{\theta_t}^{\lin} - \mathbf{f}_{\theta_t^{\lin}}^{\lin}\right)\right) \mathbf{z} ds \\
    \ge a_{\mathcal{C}} \left\lVert \mathbf{z} \right\rVert_{2}^2.
\end{equation}
Now,
\begin{align}
    \frac{d}{dt} \frac{1}{2}\lVert \mathbf{f}_{\theta_t}^{\lin} - \mathbf{f}_{\theta_t^{\lin}}^{\lin} \rVert_{\hat{\mathbf{\Theta}}_0^{-1}}^2 
    =& -\eta_0 \left( \mathbf{f}_{\theta_t}^{\lin} - \mathbf{f}_{\theta_t^{\lin}}^{\lin} \right)^{\top} \hat{\mathbf{\Theta}}_0^{-1} \mathbf{J}_{\theta_0} \Delta_t \\
    &- \eta_0 \left( \mathbf{f}_{\theta_t}^{\lin} - \mathbf{f}_{\theta_t^{\lin}}^{\lin} \right)^{\top} \hat{\mathbf{\Theta}}_0^{-1}  \hat{\mathbf{\Theta}}_0 \overline{\mathbf{H}}_t (\mathbf{f}_{\theta_t}^{\lin} - \mathbf{f}_{\theta_t^{\lin}}^{\lin}).
\end{align}
We can bound the first term using Cauchy-Schwarz and $\lVert  \mathbf{J}_{\theta_0} \Delta_t \rVert_{\hat{\mathbf{\Theta}}_0^{-1}} \le \lVert \Delta_t \rVert_2$:
\begin{align}
    -\eta_0 \left(  \mathbf{f}_{\theta_t}^{\lin} - \mathbf{f}_{\theta_t^{\lin}}^{\lin} \right)^{\top} \hat{\mathbf{\Theta}}_0^{-1}  \mathbf{J}_{\theta_0} \Delta_t
    &\le \eta_0 \lVert \mathbf{f}_{\theta_t}^{\lin} - \mathbf{f}_{\theta_t^{\lin}}^{\lin} \rVert_{\hat{\mathbf{\Theta}}_0^{-1}} 
    \lVert  \mathbf{J}_{\theta_0} \Delta_t \rVert_{\hat{\mathbf{\Theta}}_0^{-1}} \\ 
    &\le \eta_0 \lVert \mathbf{f}_{\theta_t}^{\lin} - \mathbf{f}_{\theta_t^{\lin}}^{\lin} \rVert_{\hat{\mathbf{\Theta}}_0^{-1}} K^{\Delta} \frac{(\log n)^c}{\sqrt{n}}.
\end{align}
For the second term, using $a_{\mathcal{C}}>0$,
\begin{equation}
\left( \mathbf{f}_{\theta_t}^{\lin} - \mathbf{f}_{\theta_t^{\lin}}^{\lin}\right)^{\top}  \overline{\mathbf{H}}_t \left(\mathbf{f}_{\theta_t}^{\lin} - \mathbf{f}_{\theta_t^{\lin}}^{\lin}\right) 
    \ge a_{\mathcal{C}} \lVert \mathbf{f}_{\theta_t}^{\lin} - \mathbf{f}_{\theta_t^{\lin}}^{\lin} \rVert_2^2 
    \ge a_{\mathcal{C}} \frac{\lambda_{\min}(\mathbf{\Theta})}{2} \lVert \mathbf{f}_{\theta_t}^{\lin} - \mathbf{f}_{\theta_t^{\lin}}^{\lin} \rVert_{\hat{\mathbf{\Theta}}_0^{-1}}^2 .
\end{equation}
In total we get
\begin{equation}
    \frac{d}{dt} \lVert \mathbf{f}_{\theta_t}^{\lin} - \mathbf{f}_{\theta_t^{\lin}}^{\lin} \rVert_{\hat{\mathbf{\Theta}}_0^{-1}}
    \le \eta_0 K^{\Delta} \frac{(\log n)^c}{\sqrt{n}}
    - \eta_0 a_{\mathcal{C}} \frac{\lambda_{\min}(\mathbf{\Theta})}{2} \lVert \mathbf{f}_{\theta_t}^{\lin} - \mathbf{f}_{\theta_t^{\lin}}^{\lin} \rVert_{\hat{\mathbf{\Theta}}_0^{-1}},
\end{equation}
and thus by Gronwall's inequality,
\begin{equation}
    \lVert  \mathbf{f}_{\theta_t}^{\lin} - \mathbf{f}_{\theta_t^{\lin}}^{\lin} \rVert_{\hat{\mathbf{\Theta}}_0^{-1}}
    \le \frac{(\log n)^c}{\sqrt{n}} \frac{2K^{\Delta}}{a_{\mathcal{C}} \lambda_{\min}(\mathbf{\Theta})}(1-e^{-\frac12 a_{\mathcal{C}} \lambda_{\min}(\mathbf{\Theta})\eta_0 t}).
\end{equation}
Thus,
\begin{equation}
    \lVert \mathbf{f}_{\theta_t}^{\lin} - \mathbf{f}_{\theta_t^{\lin}}^{\lin} \rVert_2
     \le \sqrt{2\lambda_{\max}(\mathbf{\Theta})} \frac{(\log n)^c}{\sqrt{n}} \frac{2K^{\Delta}}{a_{\mathcal{C}} \lambda_{\min}(\mathbf{\Theta})}.
\end{equation}
The bound for arbitrary test points follows by integrating (\ref{eq: functionspacedeltasplitnormal}).
\end{proof}

\subsection{Extension for cross-entropy without a reference class}
\label{appendix: extensionsoftmaxnoreference}
In this section we consider $\beta=0$ and the standard cross-entropy loss without a reference class, $\mathcal{C} = \mathcal{C}_{\CE}$, when the target probabilities have full support. It satisfies Assumptions \ref{ass: Frechet bounded}, \ref{assumption: upperboundedfunctionspacehessian} and \ref{assumption: functionspaceplcondition} (see Appendix \ref{appx: cce}). Thus it satisfies the lazy training property from Theorem \ref{thm: expdecaybetazero}. However, as the function-space loss Hessian vanishes on the $N$-dimensional subspace $\mathrm{range}(I_N \otimes 1_K)$ we cannot directly apply Theorem \ref{thm: closenesslinbetazero}. We can however show that we are close to the linearized version when projecting both to a subspace along which the Hessian does have positive eigenvalues.

Define $U := I_N \otimes 1_K \in \mathbb{R}^{NK\times N}$ and the block-centering projection
\begin{equation}
    \mathbf{P} 
    := I_{NK} - U (U^{\top} U)^{-1} U^{\top}
    = I_N \otimes \left(I_K - \frac{1}{K}1_K1_K^{\top} \right) \in \mathbb{R}^{NK\times NK}.
\end{equation}
This is a block diagonal whose $N$ blocks are the projection $P:= I_K - \frac{1}{K}1_K1_K^{\top} \in \mathbb{R}^{K\times K} $ onto $1^{\perp}$.
For standard cross-entropy loss, $\mathcal{C}(\mathbf{z}) = \mathcal{C}(\mathbf{P}\mathbf{z})$, and thus $\nabla_{\mathbf{z}}\mathcal{C}(\mathbf{z}) 
= \mathbf{P} \nabla_{\mathbf{z}}\mathcal{C}(\mathbf{z}) 
= \nabla_{\mathbf{z}}\mathcal{C}(\mathbf{P}\mathbf{z})$, and $\nabla_{\mathbf{z}\mathbf{z}}^2 \mathcal{C}(\mathbf{z}) = \mathbf{P} \nabla_{\mathbf{z}\mathbf{z}}^2 \mathcal{C}(\mathbf{z}) \mathbf{P}$.

\begin{theorem}
Let $\beta = 0$. Consider the function-space cross-entropy loss without a reference class, and assume the target probabilities have full support, i.e. $\mathbf{p}_{ik} \in (0,1)$. 
There is a constant $C_2>0$ such that for large enough layer width $n$: With probability $1-\delta_0$ over random initialization $\theta_0$,
\begin{equation}
    \forall x\in M^d, t\ge 0: 
    \quad
    \left\lVert P \left(f_{\theta_t}(x) - f_{\theta_t^{\lin}}^{\lin}(x) \right) \right\rVert_2 
    \le C_2 \frac{(\log n)^c}{\sqrt{n}}.
\end{equation}
\end{theorem}
Due to $\softmax(z) = \softmax(Pz)$, this implies that 
\begin{equation}
    \forall x \in M^d,t\ge 0:
    \quad
    \left\lVert \softmax\left( f_{\theta_t}(x) \right) - \softmax\left( f_{\theta_t^{\lin}}^{\lin}(x) \right) \right\rVert_2 \le \frac{1}{2} C_2 \frac{(\log n)^c}{\sqrt{n}}.
\end{equation}
\begin{proof}
Intuitively the theorem follows from considering the projected network $\tilde{f} := P f$, whose gradient flow is driven by the projected NTK $\tilde{\Theta} := P \Theta P$. This is not invertible, but has the pseudo inverse $\tilde{\Theta}^{\dagger} = \Theta^{-1} (I_{NK} - U (U^{\top} \Theta^{-1}U)^{-1} U^{\top} \Theta^{-1})$. For completeness, we present the entire proof.
As in the proof of Theorem \ref{thm: closenesslinbetazero}, use the triangle inequality to write
\begin{equation}
\left\lVert P \left(f_{\theta_t}(x) - f_{\theta_t^{\lin}}^{\lin}(x) \right) \right\rVert_2
\le \left\lVert P \left(f_{\theta_t}(x) - f_{\theta_t}^{\lin}(x) \right) \right\rVert_2
+ \left\lVert P \left(f_{\theta_t}^{\lin}(x) - f_{\theta^{\lin}_t}^{\lin}(x) \right) \right\rVert_2
\end{equation}
By Lemma \ref{lemma: fflindistancesametheta} we can bound the first term:
\begin{equation}
    \left\lVert P \left(f_{\theta_t}(x) - f_{\theta_t}^{\lin}(x) \right) \right\rVert_2
    \le \left\lVert f_{\theta_t}(x) - f_{\theta_t}^{\lin}(x) \right\rVert_2
    \le \frac{(\log n)^c}{\sqrt{n}} K_2' R^2.
\end{equation}
Using Lemma \ref{lemma: dthetaresgn} gives (after multiplying both sides with $J_{\theta_0}(x)$ to get into function-space)
\begin{equation}
\label{eq: functionspacedeltasplit}
    \frac{d}{dt} \left( f_{\theta_t}^{\lin}(x) - f_{\theta_t^{\lin}}^{\lin}(x)  \right)
    = -\eta_0 J_{\theta_0}(x)\Delta_t
    - \eta_0 J_{\theta_0}(x)\mathbf{J}_{\theta_0}^{\top} \overline{\mathbf{H}}_t (\mathbf{f}_{\theta_t}^{\lin} - \mathbf{f}_{\theta_t^{\lin}}^{\lin}).
\end{equation}
In particular, in the training points:
\begin{equation}
    \frac{d}{dt} (\mathbf{f}_{\theta_t}^{\lin} - \mathbf{f}_{\theta_t^{\lin}}^{\lin})
    = -\eta_0 \mathbf{J}_{\theta_0} \Delta_t
    -\eta_0 \hat{\mathbf{\Theta}}_0 \overline{\mathbf{H}}_t (\mathbf{f}_{\theta_t}^{\lin} - \mathbf{f}_{\theta_t^{\lin}}^{\lin}).
\end{equation}
Define the oblique projection 
\begin{equation}
    \hat{\mathbf{\Pi}} 
    := I_{NK} - U (U^{\top} \hat{\mathbf{\Theta}}^{-1} U)^{-1} U^{\top} \hat{\mathbf{\Theta}}^{-1}
    \in \mathbb{R}^{NK\times NK}.
\end{equation}
We get:
\begin{align}
    \frac{d}{dt} \frac{1}{2} \lVert \hat{\mathbf{\Pi}} (\mathbf{f}_{\theta_t}^{\lin} - \mathbf{f}_{\theta_t^{\lin}}^{\lin}) \rVert_{\hat{\mathbf{\Theta}}_0^{-1}}^2 
    =& -\eta_0 \left( \hat{\mathbf{\Pi}} (\mathbf{f}_{\theta_t}^{\lin} - \mathbf{f}_{\theta_t^{\lin}}^{\lin}) \right)^{\top} \hat{\mathbf{\Theta}}_0^{-1} \hat{\mathbf{\Pi}} \mathbf{J}_{\theta_0} \Delta_t \\
    &- \eta_0 \left( \hat{\mathbf{\Pi}} (\mathbf{f}_{\theta_t}^{\lin} - \mathbf{f}_{\theta_t^{\lin}}^{\lin}) \right)^{\top} \hat{\mathbf{\Theta}}_0^{-1} \hat{\mathbf{\Pi}}  \hat{\mathbf{\Theta}}_0 \overline{\mathbf{H}}_t (\mathbf{f}_{\theta_t}^{\lin} - \mathbf{f}_{\theta_t^{\lin}}^{\lin}).
\end{align}
We can bound the first term using Cauchy-Schwarz:
\begin{align}
    -\eta_0 \left( \hat{\mathbf{\Pi}} (\mathbf{f}_{\theta_t}^{\lin} - \mathbf{f}_{\theta_t^{\lin}}^{\lin}) \right)^{\top} \hat{\mathbf{\Theta}}_0^{-1} \hat{\mathbf{\Pi}} \mathbf{J}_{\theta_0} \Delta_t
    &\le \eta_0 \lVert \hat{\mathbf{\Pi}}  (\mathbf{f}_{\theta_t}^{\lin} - \mathbf{f}_{\theta_t^{\lin}}^{\lin}) \rVert_{\hat{\mathbf{\Theta}}_0^{-1}} 
    \lVert \hat{\mathbf{\Pi}} \mathbf{J}_{\theta_0} \Delta_t \rVert_{\hat{\mathbf{\Theta}}_0^{-1}} \\ 
    &\le \eta_0 \lVert \hat{\mathbf{\Pi}}  (\mathbf{f}_{\theta_t}^{\lin} - \mathbf{f}_{\theta_t^{\lin}}^{\lin}) \rVert_{\hat{\mathbf{\Theta}}_0^{-1}} K^{\Delta}\frac{(\log n)^c}{\sqrt{n}}.
\end{align}
For the second term, we note that by Lemma \ref{lemma: trajectoriescompactset} the trajectories of $\mathbf{f}_{\theta_t}^{\lin}$ and $\mathbf{f}^{\lin}_{\theta_t^{\lin}}$ are in a compact set (with high probability over initialization). Define $a_{\mathcal{C}}^+>0$ as the minimum of the minimum eigenvalue of $\nabla_{\mathbf{z}\mathbf{z}}^2 \mathcal{C}(\mathbf{z})$ on $\mathrm{range}(U)^{\perp}$ for $\mathbf{z}$ in this compact set.
We use $\hat{\mathbf{\Theta}}_0^{-1} \hat{\mathbf{\Pi}} \hat{\mathbf{\Theta}}_0
= \hat{\mathbf{\Pi}}^{\top}$, $\overline{H}_t = \mathbf{P} \overline{H}_t \mathbf{P}$, and $\mathbf{P}\hat{\mathbf{\Pi}} = \mathbf{P}$ to write
\begin{align}
    \left( \hat{\mathbf{\Pi}} (\mathbf{f}_{\theta_t}^{\lin} - \mathbf{f}_{\theta_t^{\lin}}^{\lin}) \right)^{\top} \hat{\mathbf{\Theta}}_0^{-1} \hat{\mathbf{\Pi}}  \hat{\mathbf{\Theta}}_0 \overline{\mathbf{H}}_t (\mathbf{f}_{\theta_t}^{\lin} - \mathbf{f}_{\theta_t^{\lin}}^{\lin})
    &=  \left( \mathbf{P} (\mathbf{f}_{\theta_t}^{\lin} - \mathbf{f}_{\theta_t^{\lin}}^{\lin}) \right)^{\top}  \overline{\mathbf{H}}_t \left( \mathbf{P}(\mathbf{f}_{\theta_t}^{\lin} - \mathbf{f}_{\theta_t^{\lin}}^{\lin})\right) \\
    &\ge a_{\mathcal{C}}^+ \lVert \mathbf{P}(\mathbf{f}_{\theta_t}^{\lin} - \mathbf{f}_{\theta_t^{\lin}}^{\lin}) \rVert_2^2  \\
    &\ge a_{\mathcal{C}}^+ \frac{\lambda_{\min}(\mathbf{\Theta})}{2} \lVert \mathbf{P}(\mathbf{f}_{\theta_t}^{\lin} - \mathbf{f}_{\theta_t^{\lin}}^{\lin}) \rVert_{\hat{\mathbf{\Theta}}_0^{-1}}^2 \\
    &\ge a_{\mathcal{C}}^+ \frac{\lambda_{\min}(\mathbf{\Theta})}{2} \lVert \hat{\mathbf{\Pi}} (\mathbf{f}_{\theta_t}^{\lin} - \mathbf{f}_{\theta_t^{\lin}}^{\lin}) \rVert_{\hat{\mathbf{\Theta}}_0^{-1}}^2.
\end{align}
In total we get
\begin{equation}
    \frac{d}{dt}\lVert \hat{\mathbf{\Pi}} (\mathbf{f}_{\theta_t}^{\lin} - \mathbf{f}_{\theta_t^{\lin}}^{\lin}) \rVert_{\hat{\mathbf{\Theta}}_0^{-1}}
    \le \eta_0 K^{\Delta}\frac{(\log n)^c}{\sqrt{n}}
    - \eta_0 a_{\mathcal{C}}^+ \frac{\lambda_{\min}(\mathbf{\Theta})}{2} \lVert \hat{\mathbf{\Pi}} (\mathbf{f}_{\theta_t}^{\lin} - \mathbf{f}_{\theta_t^{\lin}}^{\lin}) \rVert_{\hat{\mathbf{\Theta}}_0^{-1}},
\end{equation}
and thus by Gronwall's inequality,
\begin{equation}
    \lVert  \hat{\mathbf{\Pi}}  (\mathbf{f}_{\theta_t}^{\lin} - \mathbf{f}_{\theta_t^{\lin}}^{\lin}) \rVert_{\hat{\mathbf{\Theta}}_0^{-1}}
    \le \frac{(\log n)^c}{\sqrt{n}} \frac{2K^{\Delta}}{a_{\mathcal{C}}^+ \lambda_{\min}(\mathbf{\Theta})}(1-e^{-\frac12 a_{\mathcal{C}}^+ \lambda_{\min}(\mathbf{\Theta})\eta_0 t}).
\end{equation}
Thus, using $\mathbf{P} = \mathbf{P}\hat{\mathbf{\Pi}}$
\begin{equation}
    \lVert \mathbf{P} (\mathbf{f}_{\theta_t}^{\lin} - \mathbf{f}_{\theta_t^{\lin}}^{\lin}) \rVert_2
    \le \lVert \hat{\mathbf{\Pi}} (\mathbf{f}_{\theta_t}^{\lin} - \mathbf{f}_{\theta_t^{\lin}}^{\lin}) \rVert_2
     \le \sqrt{2\lambda_{\max}(\mathbf{\Theta})} \frac{(\log n)^c}{\sqrt{n}} \frac{2K^{\Delta}}{a_{\mathcal{C}}^+ \lambda_{\min}(\mathbf{\Theta})}.
\end{equation}
The bound for arbitrary test points follows by integrating (\ref{eq: functionspacedeltasplit}). Note that we have
\begin{equation}
    (I-\hat{\mathbf{\Pi}}) \frac{d}{dt} (\mathbf{f}_{\theta_t}^{\lin} - \mathbf{f}_{\theta_t^{\lin}}^{\lin})
    = - \eta_0 (I-\hat{\mathbf{\Pi}}) \mathbf{J}_{\theta_0} \Delta_t
    - \eta_0 (I-\hat{\mathbf{\Pi}}) \hat{\mathbf{\Theta}}_0 \overline{H}_t (\mathbf{f}_{\theta_t}^{\lin} - \mathbf{f}_{\theta_t^{\lin}}^{\lin})
    = - \eta_0 (I-\hat{\mathbf{\Pi}}) \mathbf{J}_{\theta_0} \Delta_t,
\end{equation}
which may drift like $O((\log n)^c/\sqrt{n})t$.
\end{proof}

\section{Proof of lemmas for the linearized dynamics}
\label{appendix: proofslinearizeddynamics}
\linearizedpredinvertibleNTK*
\begin{proof}[Proof of Lemma \ref{lemma: linearizedpredinvertibleNTK}]
In the training points
\begin{equation}
    \frac{d}{dt} \mathbf{g}_t 
    = -\eta_0 \left(\Theta_{\mathbf{x},\mathbf{x}} \nabla_{\mathbf{z}}\mathcal{C}(\mathbf{g}_t) + \beta (\mathbf{g}_t - \mathbf{g}_0) \right).
\end{equation}
Thus,
\begin{equation}
    -\eta_0 \nabla_{z}\mathcal{C}(\mathbf{g}_t) =  \Theta_{\mathbf{x},\mathbf{x}}^{-1} \left(\frac{d}{dt} \mathbf{g}_t  + \eta_0\beta (\mathbf{g}_t - \mathbf{g}_0)  \right).
\end{equation}
Plugging this in for any test point,
\begin{equation}
    \frac{d}{dt} (g_t(x) -g_0(x)) 
    = \Theta_{x,\mathbf{x}}  \Theta_{\mathbf{x},\mathbf{x}}^{-1} \left(\frac{d}{dt} \mathbf{g}_t  + \eta_0\beta (\mathbf{g}_t - \mathbf{g}_0)  \right) -\eta_0 \beta (g_t(x)-g_0(x)).
\end{equation}
Hence,
\begin{equation}
    \frac{d}{dt} \left(e^{\eta_0 \beta t} (g_t(x) -g_0(x))  \right)
    = \Theta_{x,\mathbf{x}}  \Theta_{\mathbf{x},\mathbf{x}}^{-1} \frac{d}{dt} \left( e^{\eta_0 \beta t} (\mathbf{g}_t - \mathbf{g}_0) \right).
\end{equation}
Integrating gives
\begin{equation}
    g_t(x) -g_0(x)
    = \Theta_{x,\mathbf{x}}  \Theta_{\mathbf{x},\mathbf{x}}^{-1} (\mathbf{g}_t - \mathbf{g}_0).
\end{equation}
Thus, using $\mathbf{g}_{\infty} = \Phi^{-1}(\beta f_{\theta_0}(\mathbf{x}))$ gives the first expression.

Moreover, consider $\beta>0$. Then, rearranging the function-space stationary point equation
\begin{equation}
    \Theta_{x,\mathbf{x}} \nabla_{\mathbf{z}} \mathcal{C}\left(\mathbf{g}_{\infty}\right) + \beta g_{\infty}(x)
    = \beta f_{\theta_0}(x),
\end{equation}
gives
\begin{equation}
    f^{\lin}_{\theta_{\infty}^{\lin}}(x) = f_{\theta_0}(x) - \frac{1}{\beta} \Theta_{x,\mathbf{x}} \nabla_{\mathbf{z}} \mathcal{C}\left(\mathbf{g}_{\infty}\right).
\end{equation}
The second expression follows by again plugging in $\mathbf{g}_{\infty}
= \Phi^{-1}(\beta f_{\theta_0}(\mathbf{x}))$.
\end{proof}

\end{document}